%% file: main.tex
\documentclass{article}

\PassOptionsToPackage{numbers, compress}{natbib}

\usepackage[preprint]{neurips_2020}

\usepackage[utf8]{inputenc} 
\usepackage[T1]{fontenc}    
\usepackage{hyperref}       
\usepackage{url}            
\usepackage{booktabs}       
\usepackage{amsfonts}       
\usepackage{nicefrac}       
\usepackage{microtype}      

\usepackage{amsmath,amsthm,amsfonts}
\usepackage{xcolor}
\usepackage{graphicx}
\usepackage{subcaption}

\input{definition.tex}

\newcommand{\dsa}{d}
\newcommand{\OR}{\text{OR}}

\newcommand{\iid}{iid}

\title{Towards a practical measure of interference for reinforcement learning}

\author{
  Vincent Liu\textsuperscript{1}, Adam White\textsuperscript{1}\textsuperscript{2}, Hengshuai Yao\textsuperscript{3}, Martha White\textsuperscript{1} \\
  \textsuperscript{1}University of Alberta\\
  \textsuperscript{2}DeepMind\\
  \textsuperscript{3}Huawei Technologies\\
  \texttt{\{vliu1, amw8, whitem\}@ualberta.ca, hengshuai.yao@huawei.com} 
}

\begin{document}

\maketitle

\begin{abstract}
Catastrophic interference is common in many network-based learning systems, and many proposals exist for mitigating it. But, before we overcome interference we must understand it better. In this work, we provide a definition of interference for control in reinforcement learning. We systematically evaluate our new measures, by assessing correlation with several measures of learning performance, including stability, sample efficiency, and online and offline control performance across a variety of learning architectures. Our new interference measure allows us to ask novel scientific questions about commonly used deep learning architectures. In particular we show that target network frequency is a dominating factor for interference, and that updates on the last layer result in significantly higher interference than updates internal to the network. This new measure can be expensive to compute; we conclude with motivation for an efficient proxy measure and empirically demonstrate it is correlated with our definition of interference. 
\end{abstract}

\section{Introduction}
Generalization is a key property of reinforcement learning algorithms with function approximation. It is important for an agent to generalize from previous encountered samples to a larger subset of samples which have not been seen. Generalization has been extensively studied in supervised learning, where we normally assume that we can sample \iid\ inputs from a fixed input distribution and the targets are sampled from a fixed conditional distribution. 

The assumption of \iid\ inputs, however, does not hold in general. When learning on a correlated stream of data, as in RL, the learner might fit the learned function to recent data and potentially overwrite or forget previously learned information. This issue is called \emph{catastrophic interference}. Interference occurs even in the \iid\ prediction setting: an update on some set of states is said to interfere with predictions in another state when that update decreases accuracy for that state. This interference is catastrophic if it causes significant forgetting, which is typically only observed with temporally correlated data, such as in RL~\citep{bengio2020interference,goodrich2015neuron,liu2019utility}
or in the sequential multi-task learning setting~\citep{kirkpatrick2017overcoming, riemer2018learning}. 
The conventional wisdom is that catastrophic interference is particularly problematic in the control setting in RL, even single-task RL, because 
(a) when an agent explores, it receives a sequence of observations, which are likely to be temporally correlated;
(b) the agent is changing its policy while learning, making the sequence of observations non-stationary; and 
(c) the agent uses its own estimates as targets (as in temporal difference learning), which makes the target outputs non-stationary. 

It is as yet difficult to verify this conventional wisdom, as we do not have effective means to measure interference. It is commonly held that replay, target networks and the choice of representation~\citep{liu2019utility} all mitigate interference, and so improve performance. But, without a clear definition and way to measure interference in RL, it is hard to test these hypotheses.
There has been work quantifying interference for supervised learning~\citep{chaudhry2018riemannian,fort2019stiffness,kemker2018measuring,riemer2018learning}, with some empirical work even correlating catastrophic forgetting and properties of task sequences in supervised learning~\citep{nguyen2019toward}. 
In prediction, however, the definition of interference is relatively straightforward: interference corresponds to decreases in prediction accuracy, which can be measured using a stored test set. This definition, unfortunately, does not extend to the control setting: if we use value function accuracy, then we have a changing performance measure as the policy changes. 
Several papers have investigated generalization and transfer in RL ~\citep{cobbe2018quantifying,farebrother2018generalization,packer2018assessing,rajeswaran2017towards}, demonstrating that learning on new environments results in drops in performance on previously learned environments~\citep{cobbe2018quantifying}, or re-initialization can help a plateaued agent make further progress~\citep{fedus2020catastrophic}. These works, however, do not directly measure levels of interference, and instead focus on test performance on new environments or new segments of environments. 

In this paper, we propose a definition of interference for control in RL using an existing performance measure, called the Optimality Residual (OR). The interference is defined as the change in OR, with two statistics to reflect the presence of catastrophic interference.
We evaluate of our interference measures by computing the correlation with several performance metrics, including sample efficiency and stability. We also use these measures to investigate the role of common deep RL techniques, including target networks, experience replay buffer size, mini-batch size, network size, and interference in different layers. 
It is difficult---or in some cases impossible---to estimate this exact interference measure. We provide an approximation, by deriving an upper bound on the OR, and demonstrate empirically that the approximation is strongly correlated with the exact interference.

\section{Background}

In reinforcement learning (RL), an agent interacts with its environment, receiving observations and selecting actions to maximize a reward signal. We assume the environment can be formalized as a Markov decision process (MDP). An MDP is a tuple $(\States, \Actions, \text{Pr}, R, \gamma)$ where $\States$ is a set of states, $\Actions$ is an set of actions, $\text{Pr}:\States\times\Actions\times\States\to[0,1]$ is the transition probability, $R:\States\times\Actions\times\States\to\RR$ is the reward function, and $\gamma \in [0,1]$ a discount factor. 
The goal of the agent is to find a policy $\pi:\States\times\Actions\to[0,1]$ to maximize the expected discounted sum of rewards. 

Given a fixed policy $\pi$, the action-value function $ \Qpi: \States\times\Actions\to\RR$ is defined as $ \Qpi(s,a) \defeq \mathbb{E}[ \sum_{k=0}^\infty \gamma^k R_{t+k+1} | S_t = s, A_t=a ]$, 
where $R_{t+1}$ denotes the reward at time $t+1$, i.e. $R_{t+1} = R(S_t,A_t,S_{t+1})$, $S_{t+1} \sim \Pr(\cdot|S_t,A_t)$, and actions are taken according to policy $\pi$: $A_{t} \sim \pi(\cdot|S_{t})$.
The optimal value function $Q^*$ is defined as
$Q^*(s,a) \defeq \sup_{\pi} Q(s,a)$, with $\pi^*$ the policy that is greedy w.r.t. $Q^*$.
The optimal value function can be obtained using the  Bellman optimality operator for action values $\bellman:\RR^{|\States|\times|\Actions|}\to\RR^{|\States|\times|\Actions|}$:
\begin{align*}
    (\bellman Q)(s,a) \defeq \sum_{s'\in\States}\Pr(s'| s,a)\left[R(s,a,s') + \gamma\max_{a'\in\Actions} Q(s',a')\right]
\end{align*}
$Q^*$ is the unique solution of the Bellman equation $\bellman Q = Q$. Q-learning is built on this operator, iteratively updating to find the fixed point of the Bellman optimality operator. 

We can use neural networks to learn an approximation to the optimal action-value.
For $Q_\thetavec$ the approximation, with parameters $\thetavec$, the online update for non-linear semi-gradient Q learning is 
\begin{align*}
    \thetavec_{t+1} &\gets \thetavec_{t} + \alpha \delta_t \nabla_{\thetavec_{t}} Q_{\thetavec_{t}}(S_{t}, A_{t})
    \ \ \ \ \ \ \ \ \text{where }  \delta_t \defeq R_{t+1} + \gamma \left[\max_{a' \in \mathcal{A}} Q_{\thetavec_{t}}(S_{t+1}, a') - Q_{\thetavec_{t}}(S_t, A_t)\right] .
\end{align*}
%
This update with NNs typically leads to unstable performance, so is often augmented with experience replay~\citep{lin1993reinforcement} and target networks, introduced in DQN~\citep{mnih2015human}. Replay consists of storing transitions in a buffer $D$, and performing mini-batch updates sampled from this buffer, per step. Target networks use an older set of parameters $\bar \thetavec$ for $\max_{a' \in \mathcal{A}} Q_{\bar\thetavec_{t}}(s', a')$, to make the update target more stationary. 

\section{A Simple Example Relating Interference and Control Performance}
Before discussing our definition and measure of interference, it is useful to use a controlled setting to illustrate how algorithmic choices impact interference and performance. For example, we expect agents with poor representations to suffer from more interference. 
If we have a very good hand-designed, sparse representation---such as tile-coding---we expect much less interference than a neural network representation that generalizes aggressively. We use three such agents for demonstration: Q-learning with tile-coding, DQN with the Adam optimizer and DQN with the RMSprop optimizer. 

The controlled environment, called Two-Rooms, consists of two open rooms with different start and goal states. The trick is that in the first room the agent should navigate up and to the right, and in the second room down and to the left. The input state contains the xy position of the agent, and which room the agent is in. The tile coding agent represents each room independently, whereas DQN is free to generalize across rooms. The agent begins life in one room and trains just long enough (10k steps) to learn a near optimal policy. Then the agent is placed in the second room and trained much longer (70k steps) than required---to the point that over specialization is possible. Finally, the agent is placed back to learn in the first room, to evaluate the impact of extended training in the second room. 

In Figure \ref{fig:tworooms-demo}, we show online learning curves and the corresponding interference (defined in Section \ref{sec:interference_over_time}) in each room separately. Generally, we can see that when the agent is learning, there is interference; the key issue is whether learning in one room interfere with the other. The tile-coding representation---with no features shared between rooms---has
no interference in one room, while training in the other. The performance of the DQN agents drops when transfering from room 2 back to room 1. The interference is catastrophic: the agent using RMSProp does not recover the optimal policy, and the agent using Adam learns more slowly than starting from scratch. 
\begin{figure}[t]
    \centering
    \begin{subfigure}[t]{0.32\textwidth}
        \centering
        \includegraphics[width=\textwidth]{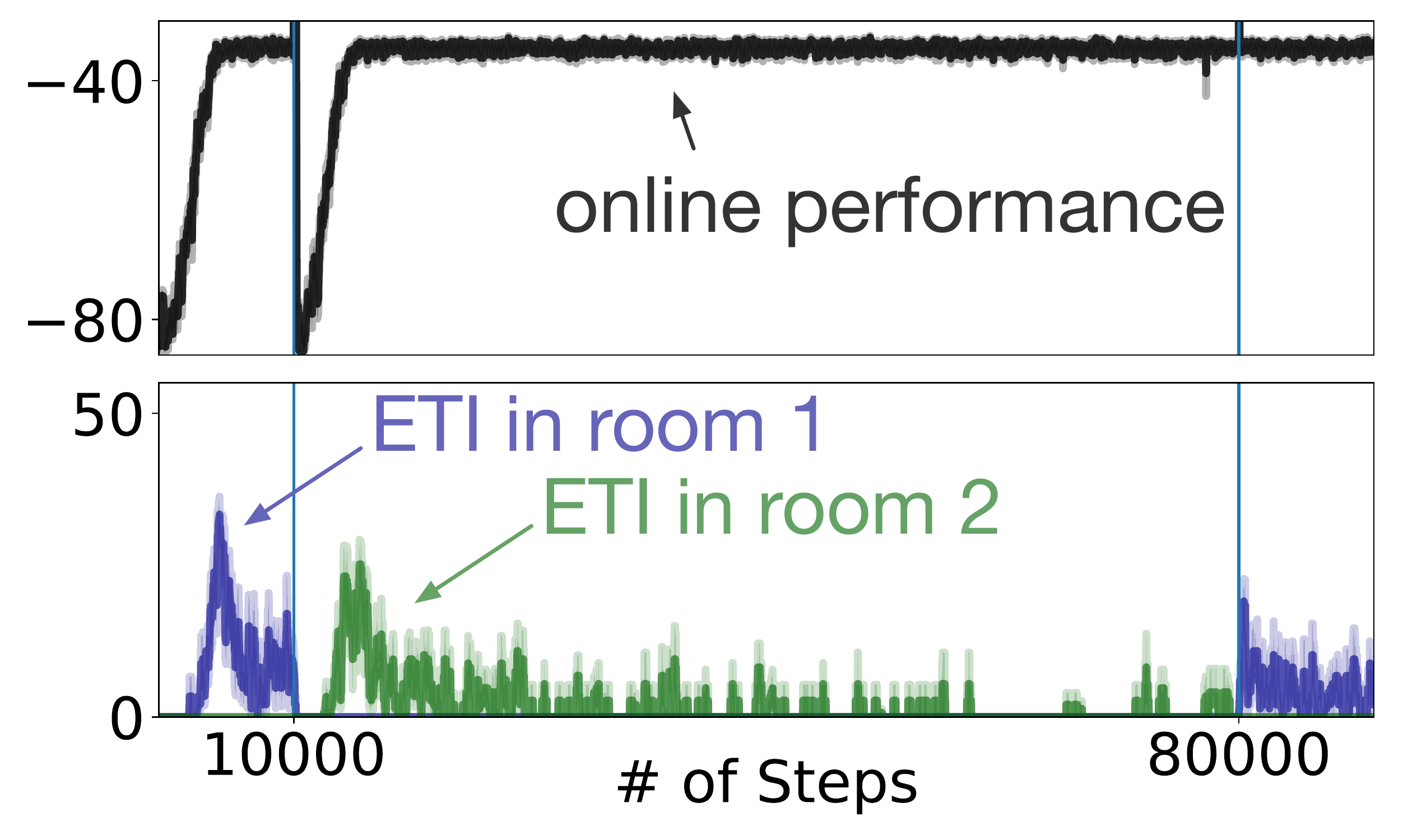}
        \caption{Tile Coding}
    \end{subfigure}
    \begin{subfigure}[t]{0.32\textwidth}
        \centering
        \includegraphics[width=\textwidth]{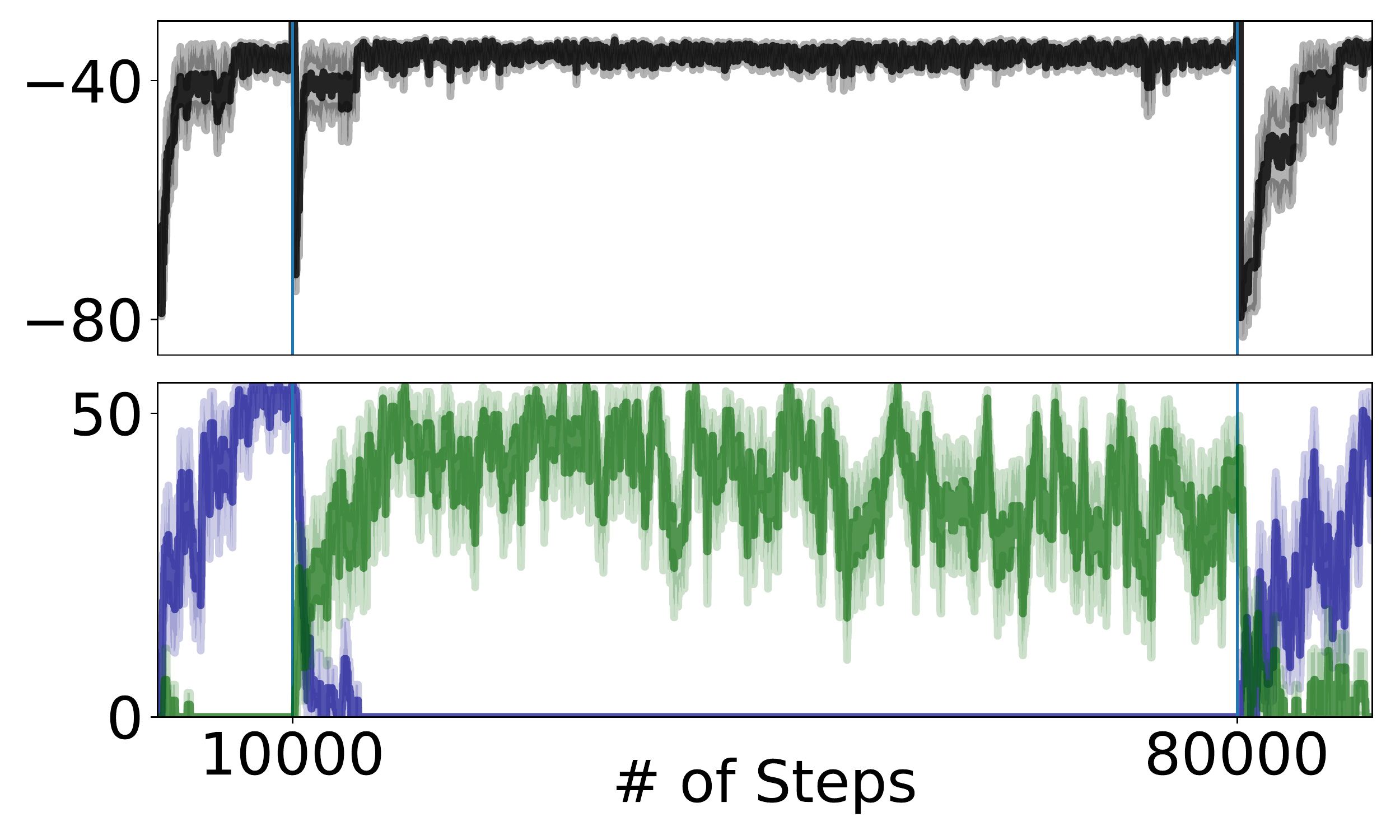}
        \caption{DQN with Adam}
    \end{subfigure}
    \begin{subfigure}[t]{0.32\textwidth}
        \centering
        \includegraphics[width=\textwidth]{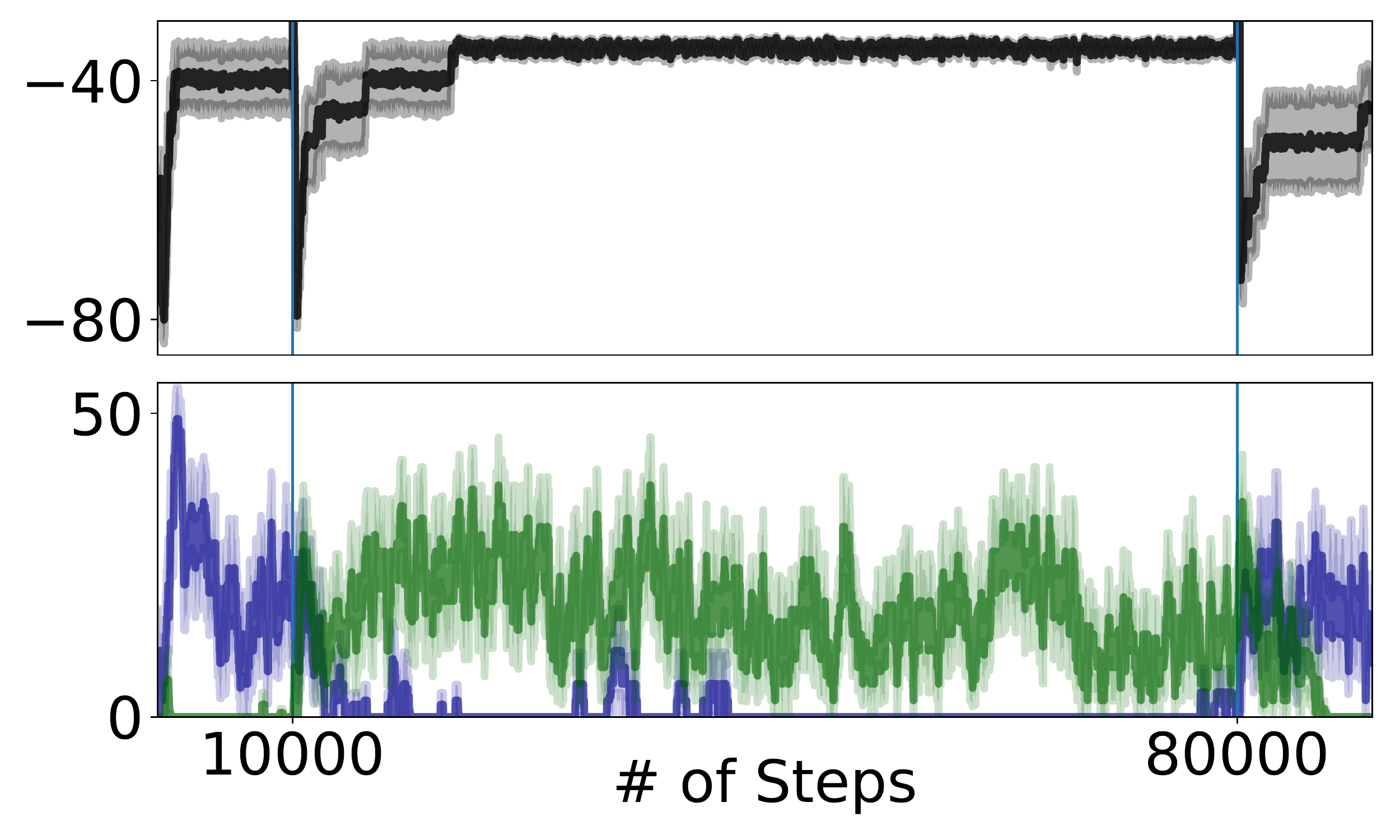}
        \caption{DQN with RMSProp}
    \end{subfigure}
    \caption{The Two-Room example. We plot the learning curve of Q-learning with different architecture choices. The three stages are indicated by the two vertical lines. ETI is a measure of interference, which is defined in a latter section. The curves are averaged over 10 runs with one standard error.}
    \label{fig:tworooms-demo}
    \vspace{-0.4cm}
\end{figure}

\section{Measuring Interference in RL}

In this section, we define interference for control in RL.
We start by discussing the definition of interference in RL for the prediction setting, where we learn $\Qpi$; we do this for clarity and to provide a contrast to the control setting. We highlight that to define whether an update causes interference requires an answer to the question: interference according to what objective? We propose a natural choice for control: the distance to the optimal action-value function.
We discuss two ways to summarize interference over time, to gauge whether an agent has high or low interference. 

\subsection{Interference in Prediction}

In the prediction setting, we estimate $\Qpi$ for a fixed $\pi$. A typical measure of prediction error is the mean-squared value error (MSVE), with state-action weighting $\dsa: \States \times \Actions \rightarrow [0, \infty)$
\begin{align*}
\text{MSVE}(\thetavec) &\defeq \| \Qpi - Q_\thetavec\|_d^2 
= {\textstyle  \sum_{s \in \States, a \in \Actions}} \dsa(s,a) (\Qpi(s,a) - Q_\thetavec(s,a))^2
\end{align*}
To quantify \emph{expected interference}, we can look at the difference in MSVE before and after an update:
$\text{MSVE}(\thetavec_{t+1}) - \text{MSVE}(\thetavec_t)$.
If this value is positive, the update generally degraded performance and there was more interference on average than positive generalization. 
If this value is negative, the update generally improved performance and there was more positive generalization than interference.

There are existing interference measures based on gradient similarity that could be used for the prediction setting. To see why, assume we can directly minimize the MSVE and so have loss  $L(\thetavec, s, a) = \tfrac{1}{2}(\Qpi(s,a) - Q_{\thetavec}(s,a))^2$. If we perform an update using $(s_t,a_t)$
\begin{align*}
    \thetavec_{t+1}
    &= \thetavec_{t} - \alpha \nabla_{\thetavec} L(\thetavec_t, s_t, a_t) 
    = \thetavec_{t} + \alpha (\Qpi(s_t,a_t) - Q_{\thetavec_t}(s_t,a_t) )\nabla_{\thetavec} Q_{\thetavec_t}(s_t,a_t)
\end{align*} 
then the interference of that update to one point $(s,a)$ is $L(\thetavec_{t+1}, s, a) - L(\thetavec_t, s, a)$. 
Using a Taylor series expansion, we get the following approximation assuming we have a small step-size $\alpha$:
\begin{equation*}
    L(\thetavec_{t+1}; s,a) - L(\thetavec_{t}; s,a)
     \approx \nabla_{\thetavec} L(\thetavec_{t}; s_t,a_t)^\top (\thetavec_{t+1} - \thetavec_{t}) 
     = -\alpha \nabla_{\thetavec} L(\thetavec_{t}; s_t,a_t)^\top \nabla_{\thetavec} L(\thetavec_{t}; s,a) 
\end{equation*}
This approximation corresponds to \emph{gradient alignment}, which has been used to learn neural networks that are more robust to interference \cite{lopez2017gradient,riemer2018learning}. They measure if $\nabla_\thetavec L(\thetavec_{t}; s_t,a_t)^\top \nabla_{\thetavec} L(\thetavec_{t}; s,a) > 0$, to determine if there is positive generalization between two samples; they generally encourage these dot-products to be positive. Other work used gradient cosine similarity, to measure the level of transferability between tasks \citep{du2018adapting}, and to measure the level of interference between objectives \citep{schaul2019ray}.
A somewhat similar measure was used to measure generalization in reinforcement learning \citep{achiam2019towards,bengio2020interference}, using the dot product of the gradients of Q functions $\nabla_{\thetavec} Q_{\thetavec_{t}}(s_t,a_t)^\top \nabla_{\thetavec} Q_{\thetavec_{t}}(s,a)$. This is related in the sense that, for the MSVE with $\delta_t = \Qpi(s_t,a_t) - Q_{\thetavec_t}(s_t,a_t)$,
 $\nabla_\thetavec L(\thetavec_{t}; s_t,a_t)^\top \nabla_{\thetavec} L(\thetavec_{t}; s,a)
    = \delta_t \delta_i \nabla_{\thetavec} Q_{\thetavec_t}(s_t,a_t)^\top \nabla_{\thetavec} Q_{\thetavec_t}(s,a)$.
This measure neglects the direction of the gradients, and so measures both positive generalization as well as interference.

In all the above, interference is measured relative to a chosen performance objective. This performance objective could even be different than the objective directly optimized by the agent. For example, the agent could optimize the MSPBE, as is done by TD-learning, and performance measured with MSVE. We could also have chosen to define the interference using the MSPBE as the performance objective. This is all to say that defining interference is relative to many givens: we need to clearly specify our performance objective, the update for the weights and what samples are used in that update. The same nuance arises in the control setting, which we discuss next. 
 
\subsection{Interference in Control}

Given a value estimation $Q_{\thetavec}$, let $\pi_{\thetavec}$ be the policy with respect to the current estimation $Q_{\thetavec}$. For example, $\pi_{\thetavec}$ can be the greedy policy w.r.t. $Q_{\thetavec}$. A previously proposed measure~\citep{farahmand2011regularization,williams1993tight} for the quality of a policy is the distance between the action-values for that policy and the optimal action-values
\begin{align*}
    \OR(\thetavec) 
    &\defeq {\textstyle \sum_{s \in \States, a \in \Actions}} \dsa(s,a) | Q^*(s,a) - Q^{\pi_{\thetavec}}(s,a) | 
    = \EE_d[ Q^*(S,A) - Q^{\pi_{\thetavec}}(S,A) ].
\end{align*}
We call this the \emph{Optimality Residual} (OR). The distribution $\dsa$ specifies the importance of a state-action pair in the OR. Often, it corresponds to the sampling distribution. For example, $d(s,a) = \nu(s) u(a | s)$ where $\nu$ is a start-state distribution and $u$ is a behavior policy. Notice that the absolute value is not included in the second line, because $Q^*(s,a) \ge Q^{\pi_{\thetavec}}(s,a)$ for all policies. 
This objective is one appropriate choice, because the target $Q^*(s,a)$ does not change as the policy changes. 


Once we have this objective, the definition for expected interference parallels the prediction setting
\begin{equation}
    \text{EI}(\thetavec_t, B_t)
    \defeq  \EE_d[\OR(\thetavec_{t+1}, (s,a)) - \OR(\thetavec_{t}, (s,a))] = \EE_d[Q^{\pi_{\thetavec_{t}}}(S,A) - Q^{\pi_{\thetavec_{t+1}}}(S,A)].
\end{equation}
where $B_t$ is the mini-batch of data used to update $\theta$ and $\OR(\thetavec, (s,a)) \defeq Q^*(s,a) - Q^{\pi_{\thetavec}}(s,a)$. 

When running experiments in reinforcement learning, where we have a simulator, it is in fact possible to estimate this quantity. One of the primary motivations for measuring interference is to facilitate investigation
by researchers. The OR can be estimated simply by using rollouts from a given $(s,a)$. The policy $\pi_{\thetavec_t}$ can be started from $(s,a)$ multiple times, generating multiple trajectories. These can be used to get a sample average estimate of the expected return from $(s,a)$ under $\pi_{\thetavec_t}$. This can then be repeated for $\pi_{\thetavec_{t+1}}$. The EI is the average OR across $(s,a) \sim d$.
In general, though, estimating the EI can be very expensive, because a large number of rollouts may be needed to get accurate estimates \citep{sajed2018high}. In RL experiments without simulators, it is generally not feasible. In Section \ref{sec:approximation}, we discuss a more practical approach to approximate the EI. First, though, we validate the utility of this true EI.

\subsection{Summarizing Interference over Time}
\label{sec:interference_over_time}

To determine the impact of interference on agent performance, we need to be provide summary statistics of interference over time. 
The above are instaneous interference measures, which can tell us how much interference occurred after an update. However, this interference
might have long range impacts, and so performance changes on this step might be impacted by interference many steps ago. 

A simple choice is to use an average EI over the last window of time. Unfortunately, this choice is problematic because the EI is signed. A negative EI actually indicates improvement---good generalization. An agent could oscillate between positive and negative EIs, with the average appearing to be near zero. The mean of skewed, potentially multi-modal distributions is not a particularly suitable choice, and we can consider other statistics. 

To be more systematic about the choice, let $X$ be the random variable corresponding to EI over the desired window of time. For example, if the agent has been learning for 1000 steps, and the desired window of time is all learning, then $X$ is a scalar RV with a density over the possible instantaneous EIs over this window of 1000 steps. The empirical distribution is the 1000 values of EI. 

We consider two statistics, one to measure if the agent had large interference values and the other if interference was highly variable. 
Catastrophic interference may occur even with only a few steps of very large interference; when reported as an average over time, these large values might be dominated by many small ones. Instead, we can look at the average of the top 10\% of interference values---the largest interference---over the window of time. 
If it is large, then at least 10\% of the time the agent had large interference. This type of measure has been used to measure risk, and termed Conditional Value at Risk or sometimes Expected Tail Loss. Correspondingly, we call this the Expected Tail Interference (ETI), defined as 
\begin{equation}
    \text{ETI}_{\alpha}(X) = \EE[X|X \geq \text{Percentile}_{1-\alpha}(X)]
\end{equation}
 $\text{Percentile}_{1-\alpha}(X)$  is the (1-$\alpha$)-percentile of the distribution of $X$. In our experiments, we set $\alpha=0.1$.

Finally, we can also provide a more accurate measure of variance by considering the interquartile range: the difference between the 75th and 25th percentiles. We call this the Interference Dispersion
\begin{equation}
    \text{Interference Dispersion}(X) = \text{Percentile}_{0.75}(X) - \text{Percentile}_{0.25}(X).
\end{equation}
Previous work~\citep{chan2020measuring} has also considered using conditional value at risk and interquartile range to measure the reliability of reinforcement learning algorithms.

\section{Empirical Evaluation: Correlation between Interference and Performance}\label{sec:correlation}
In the section, we evaluate the utility of the interference measures by computing the correlation with several performance measures, including efficiency, stability and episodic return. The goal is both to validate the utility of these measures of interference---as they would not be useful if uncorrelated with performance---as well as to investigate the impact of common deep RL techniques on interference and control performance. 

\textbf{Environments} \ \ We use Two-Rooms, designed to induce interference across the rooms, and Cart-pole, in which interference has previously been shown to be problematic~\citep{goodrich2015neuron}. Two-Rooms is designed so that the agent has sufficient information to learn optimal policies for each room, but the overlap in inputs for the two rooms is likely to cause interference for standard neural network architectures. 
Cart-pole involves balancing a pole~\citep{barto1983neuronlike}. Though a simple environment, deep RL agents fail in this domain, or learn unstable policies, as we show in our experiments, and so it provides a useful setting to understand the role of interference on performance.
The agent is run a maximal number of steps: 90k for Two-Rooms and 20k for Cart-pole. We run for a fixed number of steps, rather than episode, because otherwise some agents get more environment interactions if they have long episodes. All experiments are averaged over 10 runs.

\textbf{Agents} \ \
We investigate well-known deep RL techniques to improving learning, including experience replay, mini-batch updating, Adam optimization (particularly the addition of momentum), and target networks. 
We consider networks of two hidden layers, with various number of nodes in each layer, batch sizes, buffer sizes and target network update delay. The set of each hyperparameter and other experiment details are in Appendix \ref{app_expdetails}.

\textbf{Performance Metrics}\ \
We consider four performance measures: \emph{average episodic return} (AER), \emph{consecutive stable performance}, \emph{stable AER} and \emph{sample efficiency}. The AER reflects accumulated reward by the agent, across all steps of learning. It is computed as follows. For each step $i$ during learning, the agent has an associated expected return $\bar{G}_i$: how much reward it currently gets within an episode, in expectation. This can be estimated using multiple runs, or using a recent window of returns, to get estimate $\bar{G}_i$. The AER is the average of these across the last 50\% of steps: $\text{AER} \defeq \tfrac{2}{n} \sum_{i=n/2}^n \bar{G}_i$. The AER reflects the agents performance, on average, across the second half of its lifespan. We use the second half to gauge performance, because we are interested in assessing the impact of interference in what the agent has learned.

The AER can be measured using online or offline return. An \emph{offline} $\bar{G}_i$ is an estimate of the expected return $\EE_{sa \sim d}[Q^{\pi_{\thetavec_t}}(s, a)]$ for the policy at time step $i$, measured by averaging over Monte Carlo rollouts. It asks how well the agent would perform if it freezes its policy, and no longer performs updates. 
An \emph{online} $\bar{G}_i$ is an average over the most recent episodic returns obtained by the agent online, computed using an exponential average with weighting 0.1. 

We define consecutive stable performance as the maximum number of consecutive steps above a performance threshold (60 step for Two-Rooms, 200 for Cart-pole), divided by the total number of steps. If that number is 1, the agent's threshold performance is maximally stable; if it is zero, it is maximally unstable. 
Sample complexity corresponds to the first step $i$ that the agent reaches a performance threshold for $k$ consecutive steps (we use $k = 500$), divided by the total number of steps. Sample efficiency is $1 -$ sample complexity. If the agent has less interference, we expect the agents to learn a good policy faster, though an agent that generalizes aggressively---and has high interference---might have good efficiency, but may not stably remain at this performance. 
Finally, stable AER is defined as $\beta \text{AER} + (1-\beta) E[\bar{G} | \bar{G} \le \text{Percentile}_{0.1}(\bar{G})]$ where $\beta$ represents the risk profile of the algorithm designer. If the agent has high AER but is unstable, then it will have lower stable AER under a small risk-tolerance $\beta$. 

We measure Kendall's Rank-Correlation Coefficient, as in~\citep{jiang2019fantastic}, which reflects if two different measures rank agents similarly. It is agnostic to magnitude or precise numbers: if the interference and performance measure both say agent 1 is better than agent 2, then they are reporting similar outcomes. 
See Appendix \ref{def_rank_coef} for the formula.

\textbf{Results} \ \ We show the correlation coefficients between the two interference measures, ETI and Interference Dispersion, and the above four performance measures, in Figure \ref{fig:kendall}. We expect negative correlations, since high interference should correspond to low performance. The overall conclusion is that ETI and Interference Dispersion are both negatively correlated with all performance measures, providing some evidence for the validity of these interference measures. 


\begin{figure*}[t]
    \centering
    \includegraphics[width=0.95\textwidth]{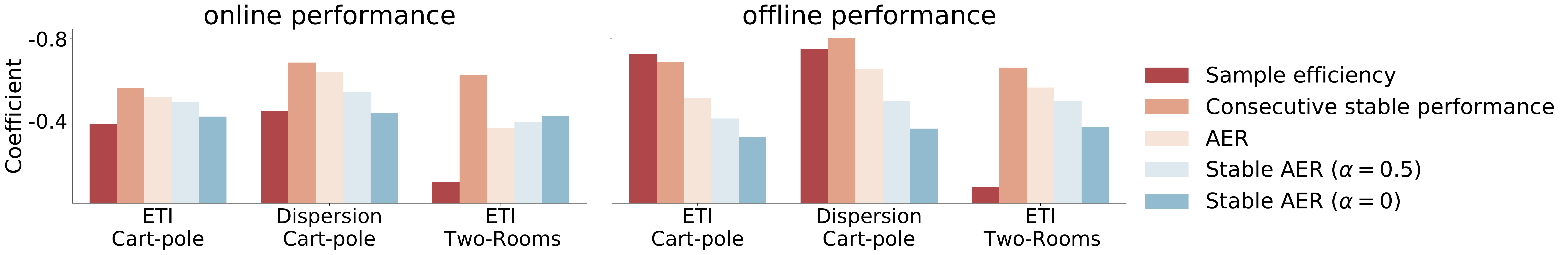}
    \caption
    {
        Kendall’s rank coefficient on Cart-pole and Two-Rooms. Interference Dispersion is close to zero in Two-Rooms and the correlation is small, so we do not report the numbers in the figure. 
    }
    \label{fig:kendall}
    \vspace{-0.5cm}
\end{figure*}

Next, we look at correlations between performance and interference, at a more fine-grained algorithmic level. To do so, we use a scatter plot for each agent, labeled based on the choice of mini-batch size, buffer size and target network update frequency. The y-axis is performance, and the x-axis interference, allowing a visual inspection of correlation between the two as well as general trends for each algorithm choice. We create one scatter plot per environment, per performance measure, and per interference measure; we include only a subset in Figure \ref{fig:scatter} and the remainder in Appendix \ref{more_experiment}. 
We find several conclusions. 
1) The batch size, buffer size and network size did not seem to have a large impact on either interference or performance; instead, target network frequency was the dominating factor. 2) The target network frequency had opposite performance in the two environments: it increases interference in Cart-pole and reduced it in Two-Rooms. In Two-Rooms, target networks improve stable performance at the cost of reducing efficiency.

Besides optimization, another important component of deep reinforcement learning is the function approximator. Therefore, we conduct an experiment to measure interference within a network, in Appendix \ref{within_network}. We find that updates on the last layer result in significantly higher interference than updates in the internal layers. The result motivates future research directions to mitigate interference: (1) strategies to mitigate interference in the last layer, and (2) algorithms to learn representation such that updating the last layer on top of these representation is robust to interference.


\begin{figure*}[ht]
    \vspace{-0.5cm}
    \centering
    \includegraphics[width=0.98\textwidth]{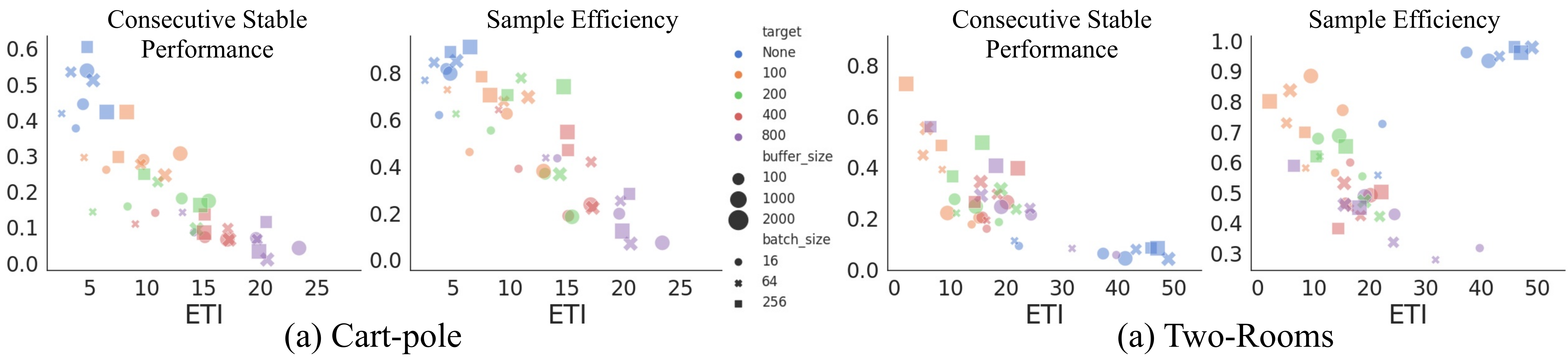}
    \caption
    {
        ETI vs sample efficiency and consecutive stable performance in Cart-pole and Two-Rooms, for a variety of Deep RL agents with the smallest network size of 128x128.
    }
    \label{fig:scatter}
    \vspace{-0.5cm}
\end{figure*}

\section{Approximating the Expected Interference with TD Errors}
\label{sec:approximation}


It can be impractical to compute the EI, and instead we will need to approximate it. One obvious strategy is simply to estimate $Q^{\pi_t}$ from sampled data, and use estimates $\hat{Q}^{\pi_{t-1}} - \hat{Q}^{\pi_t}$ from a set of sampled states, such as sampled start states. The estimate $\hat{Q}^{\pi_{t-1}}$ could be used to initialize $\hat{Q}^{\pi_t}$, so that fewer updates are needed, as likely $\pi_t$ and $\pi_{t-1}$ are not too different. Unfortunately, such a simple strategy, and ideas related to directly estimating this difference, perform poorly (see Appendix \ref{app_othermeasures}). The issue is that approximation of EI with these estimates seems highly sensitive to accuracy, and it is expensive---or impossible if there is insufficient data---to get highly accurate estimates. 

Instead, we want a proxy measure that is more likely to maintain the same sign as EI: reflect performance improvements if the agent got better, and performance degradation otherwise. A natural proxy measure is the Bellman error. The Bellman error reflects if the agent has gotten closer to a fixed point; if it reduced between steps, then this suggests the agent is closer to the fixed point and likely that there is a performance improvement. 
%
Fortunately, there is quite a lot of theory relating the Bellman error to $V^{\pi}$. We extend previous results---namely Lemma 4.3 and Theorem 5.3 in \citet{munos2007performance}---to the action-value setting. Though relatively straightforward,  modifications were needed to allow for differences in distribution over action selection from start states, particularly in the redefinition of concentration coefficients used below. We first present a lemma that upper bounds the EI in terms of the Bellman error. All proofs are in Appendix \ref{proof_lemma}.

\begin{lemma}
    Let $Q \in \RR^{|\States||\Actions|}$ and $\pi$ be a greedy policy with respect to $Q$. Then
    \begin{equation}
        (Q^* - Q^\pi) 
        \leq A |\bellman Q - Q| \hspace{0.8cm}\text{ and }  \hspace{0.8cm} d(Q^* - Q^\pi) 
        \leq d A |\bellman Q - Q|
    \end{equation}
    where $A \defeq [(\eye-\gamma P \Pi^{\pi^*})^\inv + (\eye-\gamma P \Pi^{\pi})^\inv]$, with $\frac{1-\gamma}{2} A$ a stochastic matrix.
    \label{lemma:componentwise_bound}
\end{lemma}
This bound tells us that we can sample the state-action pairs proportionally to $d A$ to upper bound the OR. Sampling according to $d A$, however, is typically infeasible and here again we need some approximation. We can usually only expect to have a sampled set of transitions, under some behavior policy, resulting in states $s$ in each transition sampled according to some $\mu: \States \rightarrow [0,\infty)$. We can additionally bound this sampling error, by using concentration coefficients. Assume $d(s,a) = \nu(s) /|\Actions|$, where implicitly actions are sample uniformly from $s$. We show the result for any $p \ge 1$ and any policies with non-zero support on all actions in Theorem \ref{theorem1}, with the informal result written here with $p=1$ and uniform policies for simplicity.

\textbf{Theorem 1.} [Informal]
    Let $\nu$ and $\mu$ be probability measures on $\States$. For $\pi$ greedy w.r.t. $Q \in \RR^{|\States||\Actions|}$
\begin{equation*}
    {\textstyle \sum_{s, a}} d(s,a) (Q^*(s,a) - Q^{\pi_t}(s,a)) \leq \tfrac{2}{1-\gamma} [C(\nu,\mu)] {\textstyle \sum_{s, a}} \tfrac{\mu(s)}{|\Actions|} | (\bellman Q_t) (s,a)- Q_t(s,a) |
    .
\end{equation*}
The concentration coefficient $C(\nu,\mu)$ reflects differences in state visitation, starting from $\nu$ versus $\mu$, defined precisely in Appendix \ref{proof_lemma}. We test three practical choices of $\mu$, with $\hat{d}(s,a) = \mu(s) /|\Actions|$. 

If this approximation is relatively good, then $\EE_{(s,a)\sim d}[Q^* - Q^{\pi_{\thetavec_t}}]$ is approximately proportional to $\EE_{(s,a)\sim \hat{d}}[|\bellman Q_{\thetavec_t} - Q_{\thetavec_t}|]$. Recall that EI is $\EE_{(s,a)\sim d}[(Q^* - Q^{\pi_{\thetavec_{t}}}) - (Q^* - Q^{\pi_{\thetavec_{t-1}}})]$. Therefore, a potentially reasonable approximation of EI using the Bellman error is $\EE_{(s,a)\sim \hat{d}}[|\bellman Q_{\thetavec_{t}} - Q_{\thetavec_{t}}|- |\bellman Q_{\thetavec_{t-1}} - Q_{\thetavec_{t-1}}|]$. Even this approximation remains difficult to sample, due to the double sampling problem for Bellman error.
%
Fortunately, we only need to approximate the difference rather than each term. This can be reasonably well approximated uses differences in TD error. Let $\delta(\thetavec; s, a, r, s') \defeq r + \gamma \max_{a' \in \Actions} Q_\thetavec(s', a') - Q_\thetavec(s,a)$. By the bias-variance decomposition~\citep{antos2008learning}, we can show that
\begin{align}
    \EE[\delta(\thetavec; s,a,r,s')^2] \nonumber
    &= \EE[|\bellman Q_{\thetavec}(s,a) - Q_{\thetavec}(s,a)|^2] + \EE[|r + \max_{a'} Q_{\thetavec}(s', a')  - \bellman Q_{\thetavec}(s,a)|^2].
    \label{eq:bias-variance}
\end{align}
The first term is the desired Bellman error, and the second term the variance of the targets. If the environment is deterministic, then this variance is zero. More generally, the \emph{Approximate EI}, using TD errors, satisfies
\begin{align*}
\text{AEI} &\defeq \EE_{(s,a)\sim \hat{d}}[\delta(\thetavec_{t}; s,a,r,s')^2 - \delta(\thetavec_{t-1}; s,a,r,s')^2] \\
&= \EE_{(s,a)\sim \hat{d}}[|\bellman Q_{\thetavec_{t}}(s,a) - Q_{\thetavec_{t}}(s,a)|^2 -|\bellman Q_{\thetavec_{t-1}}(s,a) - Q_{\thetavec_{t-1}}(s,a)|^2] \\
&+ \EE_{(s,a)\sim \hat{d}}[|r + \max_{a'} Q_{\thetavec_{t}}(s', a')  - \bellman Q_{\thetavec_{t}}(s,a)|^2 - |r + \max_{a'} Q_{\thetavec_{t-1}}(s', a')  - \bellman Q_{\thetavec_{t-1}}(s,a)|^2].
\end{align*}
The second expectation is likely to be small, because the two parameters likely have similar variances. 
 

\subsection{Choosing a Measure $\mu$ to Approximate the Expected Interference}
The quality of the approximation is heavily based on the sampling distribution $\mu$. Ideally, we want a measure $\mu$ such that the concentration coefficient $C(\nu,\mu)$ is small, though this is difficult to ascertain. We only have a stream of observations of the agent interacting with the environment, and further can likely only keep a subset of those in a buffer. Sampling from such a buffer is implicitly sampling from a measure $\mu$, where the data acts like a non-parametric sampling distribution. We can consider multiple strategies for adjusting this sampling distribution, both by choosing what to store in the buffer and by re-weighting samples obtained from the buffer, similarly to importance sampling.

We consider three practical choices. The first, which we call \emph{buffer}, involves simply sampling from the most recent transitions. The AEI is then approximated by averaging the differences in TD errors from uniformly sampled transitions from this buffer. The second strategy, which we call \emph{reservoir}, approximates uniform sampling from all the past transition,  by maintaining a reservoir buffer. The third strategy, which we call \emph{discounted}, involves reweighting transitions in the reservoir buffer. To approximately sampling from the discounted future state distribution, we re-weight each transition by $(1-\gamma)\gamma^t$ where t is the number of steps in that episode. We use re-weighting instead of sampling since we would like the measure to have smaller variance. 

\subsection{Empirical Correlation between EI and AEI}\label{sec:empirical_correlation}
\vspace{-0.1cm}
We empirically demonstrate that the approximations of interference are correlated with true interference in Two-Rooms and Cart-pole. We sample 1000 transitions, which is a relatively small number compared to the state space, and so more reflective of realistic limitations. We  measure Pearson correlation in Figure \ref{fig:correlation}, between EI and AEI per step as well as ETI and Approximate ETI, for two agents. We provide the details in Appendix \ref{app_correlation}.

Though there are several approximation steps above, we find that AEI correlate highly with EI, most clearly in Cart-pole but also in Two-Rooms. The sampling strategies are similarly effective, though reservoir sampling seems to be most effective. We also conduct the same experiments for AEI as in Section 5 are in Appendix \ref{app_correlation}, with similar conclusions, though with slightly reduced correlations to performance measures. 

\begin{figure*}[ht]
    \vspace{-0.3cm}
    \centering
    \includegraphics[width=0.8\textwidth]{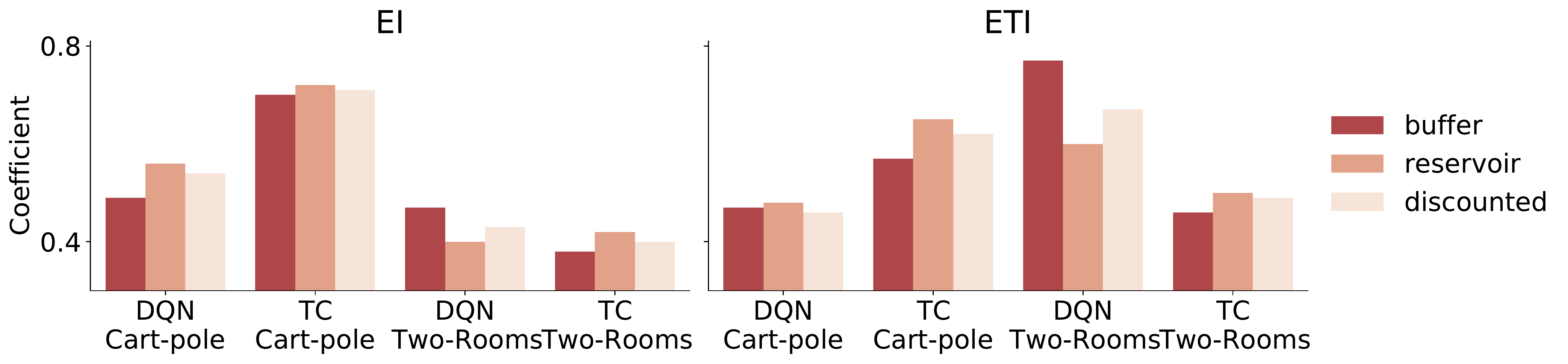}
    \caption
    {
        Correlation coefficients with EI and ETI in two domains. All result have p-value$~<~0.01$.
    }
    \label{fig:correlation}
    \vspace{-0.5cm}
\end{figure*}

\section{Conclusion}

In this paper, we propose a definition of interference for control in RL, and provide a practical approximation using TD errors. We validate the utility of the interference measures by computing the correlation with several performance metric. Using the proposed measures, we provide some insights into interference in deep reinforcement learning algorithms. We highlighted the role of the target network, which we found significantly increased interference and decreased performance in a setting where it was not needed. In another setting, however, the lack of a target network resulted in fast but unstable learning, and we found the opposite conclusion. In both cases, the correlation to interference was clear, for both the true and approximate measures.  


This is one of the first papers specifically attempting to define interference for control, and naturally has limitations. One important next step is to expand the set of environments, and agents. In this first small-scale study, we developed a methodology for such experiments, which can be leveraged to extend to new settings. Another important step is to further explore approximations to the true interference, as well as find more clear theoretical reasons why we see that the change in TD-errors performs so well as a proxy. Finally, this paper focuses on deterministic, greedy policies with learned action-values. 
There is some evidence that a mixture of policies might be more robust to interference~\citep{kakade2002approximately,vieillard2019deep}. Stochastic policies naturally fit in our definition of EI, but our approximation may not be as suitable. 

\section*{Broader Impact}
This work focuses on characterizing and understanding an RL agent's behavior. It is unlikely to have a direct impact on society although it may guide future research with such an impact. For example, future research following this work may involve developing stable and practical RL algorithms applied to real world problems.


\bibliographystyle{plainnat}
\bibliography{main.bib}

\clearpage
\appendix
\input{appendix.tex}

\end{document}

%% file: definition.tex
\newcommand{\makevector}[1]{{\mathbf #1}}

\newcommand{\rvec}{{\makevector{r}}}
\newcommand{\wvec}{{\makevector{w}}}

\newcommand{\paravec}{\beta}

\newcommand{\weights}{{\makevector{w}}}

\newcommand{\thetavec}{{\boldsymbol{\theta}}}

\newcommand{\phivec}{{\boldsymbol{\phi}}}

\newcommand{\RR}{{\mathbb{R}}}
\newcommand{\EE}{{\mathbb{E}}}

\newcommand{\wdim}{d}

\newtheorem{lemma}{Lemma}
\newtheorem{theorem}{Theorem}
\newtheorem{definition}{Definition}

\newtheorem{innercustomgeneric}{\customgenericname}
\providecommand{\customgenericname}{}
\newcommand{\newcustomtheorem}[2]{%
  \newenvironment{#1}[1]
  {%
   \renewcommand\customgenericname{#2}%
   \renewcommand\theinnercustomgeneric{##1}%
   \innercustomgeneric
  }
  {\endinnercustomgeneric}
}

\newcustomtheorem{customlemma}{Lemma}

 \newcommand{\defeq}{:=}

\newcommand{\bellman}[1]{\mathcal{T}#1}

\newcommand{\States}{\mathcal{S}}
\newcommand{\Actions}{\mathcal{A}}

\newcommand{\Qpi}{Q^\pi}

\newcommand{\eye}{\mathbf{I}}

\newcommand{\inv}{{\raisebox{.2ex}{$\scriptscriptstyle-1$}}}

%% file: appendix.tex
\section{Proofs and Technical Details}
\label{proof_lemma}
The proof of Lemma \ref{lemma1} and Theorem \ref{theorem1} are modified from \citet{munos2007performance}. To begin with, we introduce notation in a matrix form. Define the transition matrix $P \in \RR^{|\States||\Actions| \times |\States|}$ where $P(s,a,s')=\Pr(s,a,s')$. 
Given a policy $\pi$, we define $\Pi^\pi \in \RR^{|\States| \times |\States||\Actions|}$ as 
\begin{align*}
    \Pi^{\pi} = 
    \begin{pmatrix} 
    & \mathbf{\pi(s_1)} & &  \\
    & & \mathbf{\pi(s_2)} & \\
    & & & \ddots & \\
    & & & & \mathbf{\pi(s_{|\States|})} \\
    \end{pmatrix}
\end{align*}
where $\mathbf{\pi(s_i)} = \left[\pi(a_1 | s_i) \dots \pi(a_{|\Actions|} | s_i)\right]$ and all other components are zeros. This notation of $\Pi^\pi$, first introduced in \citet{wang2007dual}, is convenient to use since $\Pi^\pi P$ gives the state to state transition and $P\Pi^\pi$ gives the state-action to state-action transition.

Given an action-value function $Q$, we define the Bellman operator w.r.t. a policy $\pi$ by. 
\begin{align*}
    \bellman^\pi Q = \rvec + \gamma P \Pi^{\pi} Q.
\end{align*}
where $\rvec(s,a) \defeq \sum_{s'\in\States} \Pr(s,a,s')R(s,a,s')$ is the expected immediate reward from state $s$ after taking action $a$. Let $\pi_Q$ denote the greedy policy w.r.t. $Q$, the Bellman optimality operator is defined by
\begin{align*}
    \bellman Q = \rvec + \gamma P \Pi^{\pi_Q} Q.
\end{align*}
Since $\pi_Q$ is the greedy policy, we can show that, for any policy $\pi$,
\begin{align*}
    \bellman Q \geq \bellman^\pi Q.
\end{align*}
Here $\geq$ denotes the component-wise inequality.
Moreover, it is known that the $Q^{\pi_Q}$ is the fixed point of the operator, that is,
\begin{align*}
    \bellman Q^{\pi_Q} = Q^{\pi_Q}.
\end{align*}

\begin{customlemma}{1}\label{lemma1}
    Let $Q \in \RR^{|\States||\Actions|}$ and $\pi$ be a greedy policy with respect to $Q$. Then
    \begin{align}
        & d (Q^* - Q^\pi) \leq  d A |\bellman Q - Q|
    \end{align}
    where $A \defeq [(\eye-\gamma P \Pi^{\pi^*})^\inv + (\eye-\gamma P \Pi^{\pi})^\inv]$, with $\frac{1-\gamma}{2}A$ a stochastic matrix.
\end{customlemma}
\begin{proof}[Proof of Lemma 1]
    Using the fact that $Q^*=\bellman^{\pi^*} Q^*$, $\bellman Q \geq \bellman^{\pi^*} Q$ and $\bellman Q = \bellman^\pi Q$, we can show 
    \begin{align*}
        Q^* - Q^\pi 
        &= \bellman^{\pi^*} Q^* - \bellman^{\pi^*} Q^\pi + \bellman^{\pi^*} Q^\pi - \bellman Q + \bellman Q - \bellman^\pi Q^\pi \\
        &\leq (\bellman^{\pi^*} Q^* - \bellman^{\pi^*} Q^\pi + \bellman^{\pi^*} Q^\pi - \bellman^{\pi^*} Q) + (\bellman^\pi Q - \bellman^\pi Q^\pi) \\
        &\leq \gamma P \Pi^{\pi^*} (Q^* - Q^\pi + Q^\pi - Q) + \gamma P \Pi^{\pi}(Q - Q^\pi) \\
        &= \gamma P \Pi^{\pi^*} (Q^* - Q^\pi) + (\gamma P \Pi^{\pi^*} - \gamma P \Pi^{\pi})(Q^\pi - Q).
    \end{align*}
    Note that $(\eye-\gamma P \Pi^*)$ is invertible, so we have
    \begin{align*}
        Q^* - Q^\pi \leq (\eye-\gamma P \Pi^{\pi^*})^\inv (\gamma P \Pi^{\pi^*} - \gamma P \Pi^{\pi})(Q^\pi - Q).
    \end{align*}
    Moreover, we can derive a component-wise equality between the Bellman residual and $(Q^\pi - Q)$:
    \begin{align*}
        (\eye-\gamma P \Pi^{\pi})(Q^\pi - Q) 
        &= Q^\pi - Q - \gamma P\Pi^\pi Q^\pi + \gamma P\Pi^\pi Q \\
        &= Q^\pi - Q + \rvec + \gamma P\Pi^\pi Q - (\rvec + P\Pi^\pi Q^\pi) \\
        &= Q^\pi - Q + \bellman^\pi Q - \bellman^\pi Q^\pi \\
        &= \bellman^\pi Q - Q = \bellman Q - Q.
    \end{align*}
    Therefore, 
    \begin{align*}
        Q^* - Q^\pi 
        &\leq (\eye-\gamma P\Pi^{\pi^*})^\inv (\gamma P\Pi^{\pi^*} - \gamma P\Pi^{\pi})(\eye-\gamma P\Pi^{\pi})^\inv (\bellman Q - Q) \\
        &= (\eye-\gamma P\Pi^{\pi^*})^\inv [(\eye-\gamma P\Pi^{\pi}) - (\eye-\gamma P\Pi^{\pi^*})] (\eye-\gamma P\Pi^{\pi})^\inv (\bellman Q - Q)\\
        &= [(\eye-\gamma P\Pi^*)^\inv - (\eye-\gamma P\Pi^{\pi^{\pi}})^\inv] (\bellman Q - Q)\\
        &\leq [(\eye-\gamma P\Pi^{\pi})^\inv + (\eye-\gamma P\Pi^{\pi^*})^\inv] |\bellman Q - Q|.
    \end{align*}
\end{proof}

\begin{definition}
    Let $b$ be a policy such that $b(\cdot|s)$ has full support over the action space for all states, $\pi_1, ..., \pi_m$ be a sequence of policies, and $\nu$ and $\mu$ be two measures on $\States$. For any integer $m \geq 1$, we define 
    \begin{align*}
        c(m) \defeq \sup_{\pi_1,...\pi_m,s\in\States,a\in\Actions} \frac{(\nu \Pi^{\pi_1} P \dots P \Pi^{\pi_m})(s,a)}{\mu \Pi^b(s,a)}.
    \end{align*}
    Let $c(0) \defeq 1$ and $c(m) \defeq \infty$ if $\nu \Pi^{\pi_1} P \dots \Pi^{\pi_m}$ is not absolutely continuous w.r.t. $\mu \Pi^u$. 
    We define the discounted future state distribution concentration coefficients as
    \begin{align*}
        C(\nu,\mu, b) &\defeq (1-\gamma) \sum_{m=0}^{\infty} \gamma^m c(m).
    \end{align*}
\end{definition}

\paragraph{Remark.} 
In practice, we could choose the behavior policy $b$ as an uniform random policy or a $\epsilon$-greedy policy.

\begin{theorem} 
    \label{theorem1}
    Let $Q \in \RR^{|\States| \times |\Actions|}$, $\pi$ be a greedy policy with respect to $Q$, $u$ be an uniform policy and $b$ be a behavior policy. Let $\nu$ and $\mu$ be two probability measures on $\States$, and $d = \nu \Pi^u$. Then,
    \begin{align*}
        \sum_{s\in\States,a\in\Actions} d(s,a) |Q^*(s,a) - \Qpi(s,a)|^{p} \leq \left[\frac{2}{1-\gamma}\right]^p C(\nu,\mu,b)\sum_{s,a} \mu(s)b(a|s) |(\bellman Q)(s,a) - Q(s,a)|^{p}.
    \end{align*}
\end{theorem}
\begin{proof}[Proof of Theorem 1]
    We can write 
    \begin{align*}
        Q^* - Q^\pi \leq A |\bellman Q - Q|
    \end{align*}
    where $A=\left[(\eye-\gamma P \Pi^{\pi^*})^\inv + (\eye-\gamma P \Pi^{\pi})^\inv\right]$ and $\frac{1-\gamma}{2} A $ is a stochastic matrix. Then,
    \begin{align*}
        \sum_{s\in\States,a\in\Actions} d(s,a) |Q^*(s,a) - \Qpi(s,a)|^{p}
        &\leq \left[\frac{2}{1-\gamma}\right]^p \sum_{s,a} d(s,a) \left[\frac{1-\gamma}{2} A |\bellman Q - Q|\right]^p(s,a) \\
        &\leq \left[\frac{2}{1-\gamma}\right]^p \sum_{s,a} d(s,a) \left[\frac{1-\gamma}{2}A |\bellman Q - Q|^p\right](s,a) \\
        &\leq \left[\frac{2}{1-\gamma}\right]^p C(\nu,\mu,b) \sum_{s,a} \mu(s) b(a|s) |\bellman Q - Q|^p(s,a)
    \end{align*}
    The second inequality follows from Jensen's inequality. The third inequality follows from $d (\frac{1-\gamma}{2} A) \leq (1-\gamma) \sum_{m=0}^{\infty} \gamma^m c(m) \mu \Pi^{b} = C(\nu, \mu,b) \mu \Pi^{b}$.
\end{proof}

\section{Experimental Details}\label{app_expdetails}
\subsection{Experiment set-up}
We experiment with two environments: (1) Two-Rooms: we set the maximum steps per episode to 200, and the number of training steps to 90k, and (2) Cart-pole from OpenAI gym (\url{https://gym.openai.com/}): We set the maximum steps per episode to 500, and the number of training steps to 20k. We use a discounting factor $\gamma=0.99$ in both environments.

The environment Two-Rooms consists of two rooms with different start and goal states. In the first room the agent should navigate up and to the right, and in the second room down and to the left. The input state contains the xy position of the agent, which is in $[0,1]^2$ for both rooms, and which room the agent is in, which is in $\{0,1\}$. 

For all experiments, we use a two-layer neural network with ReLU activation, and use He intialization to initialize the neural networks.

For the experiments in Section 5, we generate a set of hyper-parameter $\Theta$ by choosing each parameter in the set:
\begin{itemize}
    \item buffer size $\in \{100, 1000, 2000\}$
    \item batch size $\in \{16, 64, 256\}$
    \item Hidden size $\in \{128, 256, 512\}$
    \item Target network update frequency $\in \{0, 100, 200, 400, 800\}$ where zero means no target network is used
\end{itemize}

For tile coding, we use 4 tiles and 16 tilings with a constant step size. The step size are searched in the set $\{0.2, 0.1, 0.05, 0.025\}$ by the best online AER. For SR-NN in Appendix \ref{within_network}, we fix $\beta = 0.1$ and use a grid search for the key parameter: $\lambda_{SKL} \in \{0.01, 0.001, 0.0001\}$.
For Section \ref{sec:empirical_correlation} and Appendix \ref{app_othermeasures}, we choose a standard neural network with hidden size of 128, batch size of 64, buffer size of 1000 and no target network. 

\subsection{Kendall’s rank-correlation coefficient}
\label{def_rank_coef}
Inspired from~\citep{jiang2019fantastic}, we use Kendall's rank-correlation coefficient~\citep{kendall1938new} to check the correlation between a performance metric and a statistics of our interference measures. 
Let $\Theta$ be a set of hyperparameters and $K \defeq \cup_{\thetavec \in \Theta} \{(g(\thetavec), s(\thetavec))\}$ where $s(\theta)$ is a statistics of  our interference measures and $g(\theta)$ is a performance measures corresponding to a  hyperparameter configuration $\theta$. Kendall’s rank coefficient $\tau$ is defined as
\begin{align*}
    \tau \defeq \frac{1}{|K||K-1|} \sum_{(g_1,s_1)\in K} \sum_{(g_2,s_2)\in K /{(g_1,s_1)}} \text{sign}(g_1 - g_2) \text{sign}(s_1 - s_2).
\end{align*}
The coefficient varies between $-1$ and $1$. 

\section{Additional Experiments of Section 5}

\subsection{Correlation between interference and performance}
\label{more_experiment}

We show the scatter plots for Cart-pole in Figure \ref{eti_cartpole} and \ref{id_cartpole}, and for Two-Rooms in Figure \ref{eti_tworooms}.
In Two-Rooms, we are interested in the performance when the agent has trained on room 2 for a long time. Therefore, we measure interference and performance for the second half of training on room 2. 

The results show show relatively consistent correlation between ETI and performance measures. The notable exception is in Two-Rooms, when there are no target networks. The agents have high interference, but also high AER, for all three variants of AER. The consecutive stable performance and sample efficiency plots sheds some light on why this occurs. Target networks slow learning in this environment, but then maintain stable performance above the threshold for consecutive stable performance. These same agents, though, look worse in terms of AER, than the No Target Network agents, which oscillate more but manage to get to higher performance. The plots are skewed by the fact that, with Target Networks, learning is not quite done when we start measuring interference, in that second half of Room 2. Consequently, though the agent is above the threshold of acceptable performance, it is still on the rise. The lower 10\% of the returns is much lower for some of the agents with target networks, than those without, because of this fact. If we allowed the agents to learn for even longer, this point that drop low on the AER plots would likely move up higher, and we would see a clear trend from the cluster of points near zero interference and high performance, the cluster of points with high interference (those without target networks).  


\subsection{Measuring interference within a network}
\label{within_network}

In deep reinforcement learning, neural networks are used as the function approximatior. We want to understand how much interference is due to the internal layers and the last layer.
The last layer typically does not have an activation; hence we can view a value function as a \emph{two-part approximation} with a representation function and a linear weight $Q_{\weights, \paravec}(s, a) \defeq \phivec_\paravec(s,a)^\top \weights$, where $\weights \in \RR^\wdim$ is the weights in the last layer and $\phivec_\paravec : \States \times \Actions \to \RR^\wdim$ is the \emph{representation} learned by the network with weights $\paravec$, composed of all the hidden layers in the network. The function $\phivec_\paravec(s,a)$ corresponds to the last layer in the network, with $\paravec$ the weights of the network. 

To study interference separately within the network, we use the stochastic block coordinate descent (SBCD) Q-learning to update $\paravec$ and $\weights$ seperately:
\begin{align*}
    \text{Representation Learning Network (RLN) updates: } & \paravec_{t+1} = \paravec_{t} - \alpha_1 \sum_{i=1}^{B} \nabla_{\paravec_{t}} L(\paravec_{t}, \wvec_t;s_i,a_i)\\
    \text{Value Learning Network updates (VLN) updates: } & \weights_{t+1} = \weights_{t} - \alpha_2 \sum_{i=1}^{B} \nabla_{\weights_t} L(\paravec_{t+1}, \weights_t;s_i,a_i).
\end{align*}
where $B$ is the mini-batch size, and $\alpha_1$ and $\alpha_2$ are learning rate. 

We measure interference for RLN and VLN updates separately at every step, and report the ETI and ETI for approximations in Table \ref{tab:seperate}. We can observe that VLN has much higher ETI than RLN. The result suggests that updates on the last layer result in significantly higher interference than updates on the internal layers, even when we decrease the learning rate for VLN. 

We include a baseline SR-NN~\citep{liu2019utility}, which learns a sparse representation $\phivec_\paravec$, to see how representation learning can reduce interference in VLN. 
SR-NN uses the distributional regularizers to learn sparse representation in neural networks: 
\begin{align*}
    \min_{\paravec} \sum_{i=1}^B L(\paravec, \weights; s_i, a_i) + \lambda_{SKL} \sum_{j=1}^{d} SKL(\bar \phivec_{\paravec,j})
\end{align*}
where $SKL$ is a regularization on the expected activation, i.e., $\bar \phivec_{\paravec,j}=\sum_{i=1}^{B} \phivec_{\paravec,j}(s_i,a_i)$ and $\phivec_{\paravec,j}$ denote the $j$-th component of $\phivec_{\paravec}$. Table \ref{tab:seperate} shows that SBCD with SR-NN has a lower ETI for VLN.


\begin{table}[h]
    \centering
    \caption{ETI for updating VLN and RLN on Cart-pole. We report the control performance as a baseline to see the magnitude of interference. Bold numbers show that ETI for VLN is significantly larger than ETI for RLN. The number are averaged over 10 runs with one standard error.}
    \begin{tabular}{l c c | c }
        \toprule
        & ETI for RLN & ETI for VLN & Control Performance\\
        \midrule
        SBCD Q-learning & 5.05 $\pm$ 0.27 & \textbf{18.23} $\pm$ 3.58 & 84.45 $\pm$ 0.76 \\
        SBCD Q-learning (smaller $\alpha_2$) & 3.83 $\pm$ 0.26 & \textbf{14.05} $\pm$ 1.59 & 86.85 $\pm$ 0.39 \\
        SBCD with SR-NN & 3.48 $\pm$ 0.21 & \textbf{5.09} $\pm$ 0.32 & 89.52 $\pm$ 0.41 \\
        \bottomrule
    \end{tabular}
    \label{tab:seperate}
\end{table}


\section{Additional Experiments of Section 6} \label{app_correlation}

\subsection{Empirical comparison of approximation strategies}
\label{app_othermeasures}

Besides approximate EI using TD errors, we test two approximation baselines. First, we could in fact directly approximate the change in Bellman error using recent insights on Kernel Bellman Errors \citep{feng2019kernel}, though the approximation is still quite expensive to compute. For example, if we use $M$ transitions to evaluate TD errors, which requires $O(M)$ computation, evaluating Kernel loss requires $O(M^2)$ computation. Hence, we use only 100 transitions (from a reservoir buffer) to evaluate the approximation.
Second, we can estimate $\hat Q^{\pi_{\thetavec_{t}}} \approx Q^{\pi_{\thetavec_{t}}}$ from sampled data in the buffer using off-policy policy evaluation (OPE), and directly approximate EI $\approx \EE_d[\hat Q^{\pi_{\thetavec_{t}}}(S,A) - \hat Q^{\pi_{\thetavec_{t+1}}}(S,A)]$ from a set of sampled state-action pairs from $d$. At each evaluation step, we run off-policy SARSA algorithm for 10 epochs over the data stored in a vanilla replay buffer.  
We call the first baseline \emph{kernel}, and the second baseline \emph{OPE}. 


During training, we compute $\text{AEI}_i$ and the true measure $\text{EI}_i$ every $k$ steps. We collect all data points $(X_i, Y_i)$ from the second half of training steps, over $10$ runs, and report Pearson correlation coefficient between AEI and EI. Formally, Pearson correlation coefficient between two sets of measures $X$ and $Y$ is defined as
\[r_{X, Y} = \frac{\sum_i (X_i - \bar X)(Y_i - \bar Y)}{\sqrt{\sum_i (X_i - \bar X)^2} \sqrt{\sum_i (Y_i - \bar Y)^2}}.\] 
We show the results in Figure \ref{fig:aei_correlation}. The results suggest that change in TD errors has higher correlation coefficients than other approximation baselines.


\begin{figure*}[ht]
    \centering
    \includegraphics[width=\textwidth]{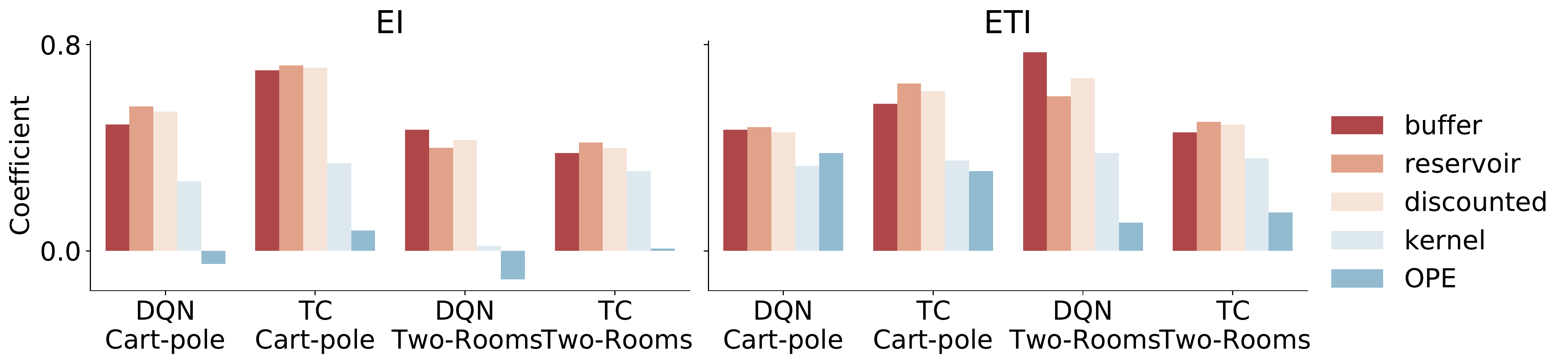}
    \caption
    {
        Correlation coefficients with AEI and Approximate ETI in two domains. 
    }
    \label{fig:aei_correlation}
\end{figure*}

\subsection{Results using AEI for Deep RL}

In this section, we present the same experiments as in Section \ref{sec:correlation} and Appendix \ref{within_network}, with the Approximate EI. We can draw similar conclusions, though with slightly reduced correlations to performance measures. Figure \ref{fig:aei_kendall} shows that Approximate ETI and ID are negatively correlated with several performance measures. Table \ref{tab:separate_aei} shows that VLN has higher Approximate ETI than RLN.

\begin{figure*}[ht]
    \centering
    \includegraphics[width=\textwidth]{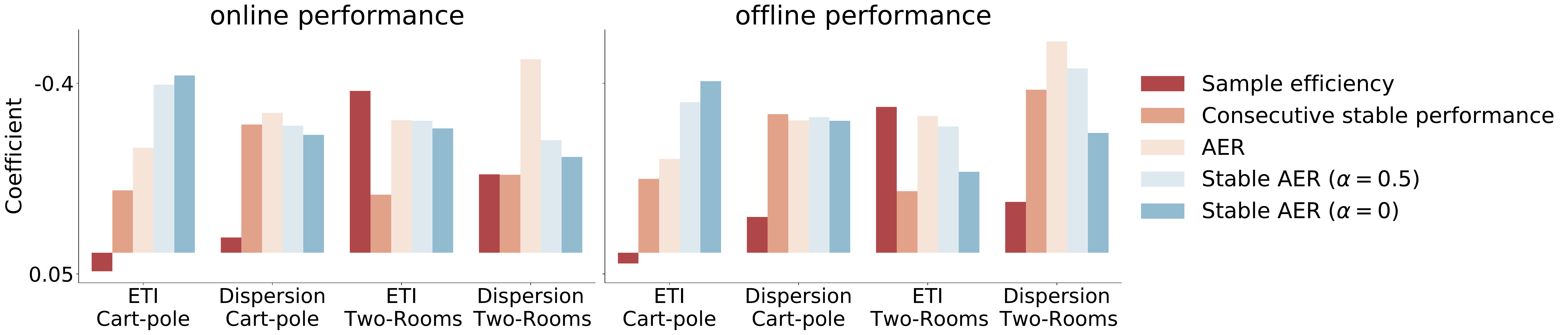}
    \caption
    {
        Kendall’s rank coefficient on Cart-pole and Two-Rooms.
    }
    \label{fig:aei_kendall}
\end{figure*}

\begin{table}[ht]
    \centering
    \begin{tabular}{l c c | c}
        \toprule
        & ETI for RLN & ETI for VLN & Performance \\
        \midrule
        SBCD Q-learning & 0.22 $\pm$ 0.02 & \textbf{1.02} $\pm$ 0.29 & 84.45 $\pm$ 0.76 \\
        SBCD Q-learning (smaller $\alpha_2$) & 0.08 $\pm$ 0.01 & \textbf{0.29} $\pm$ 0.04 & 86.85 $\pm$ 0.39 \\
        SBCD with SR-NN & \textbf{0.12} $\pm$ 0.01 & 0.08 $\pm$ 0.01 & 89.52 $\pm$ 0.41 \\
        \bottomrule
    \end{tabular}
    \vspace{0.5cm}
    \caption{Approximate ETI for updating VLN and RLN separately. Bold numbers show that ETI for VLN is significantly larger than ETI for RLN for SBCD Q-learning. For SBCD with SR-NN, ETI for RLN is larger than ETI for VLN. The number are averaged over 10 runs with one standard error.}
    \label{tab:separate_aei}
\end{table}

\begin{figure*}[ht]
    \centering
    \begin{subfigure}[t]{0.48\textwidth}
        \centering
        \includegraphics[width=\textwidth]{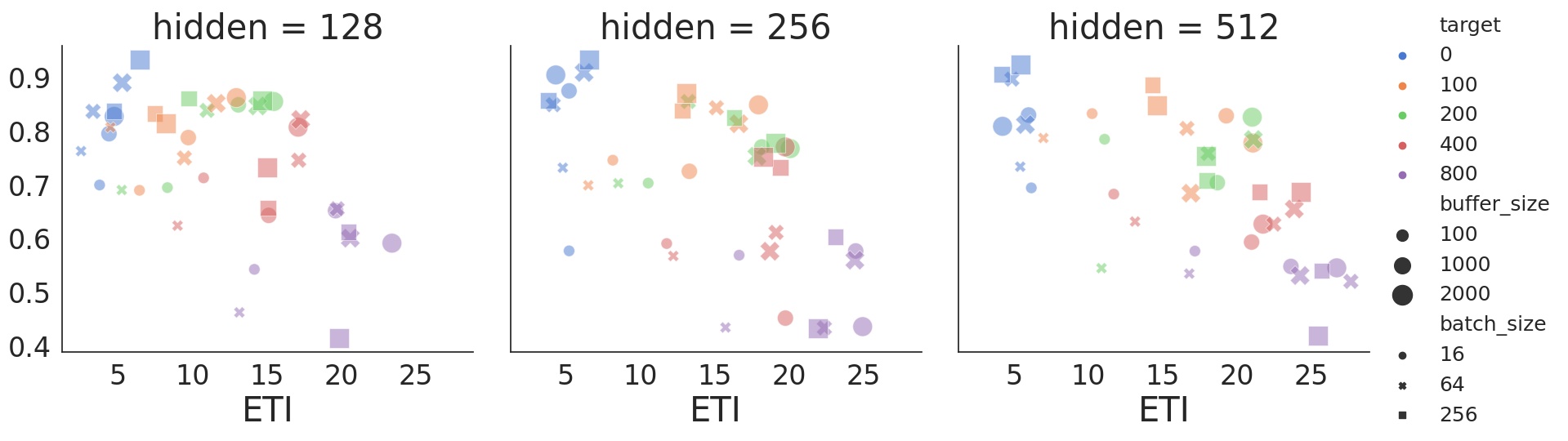}
        \caption{Online sample efficiency.}
    \end{subfigure}
    \begin{subfigure}[t]{0.48\textwidth}
        \centering
        \includegraphics[width=\textwidth]{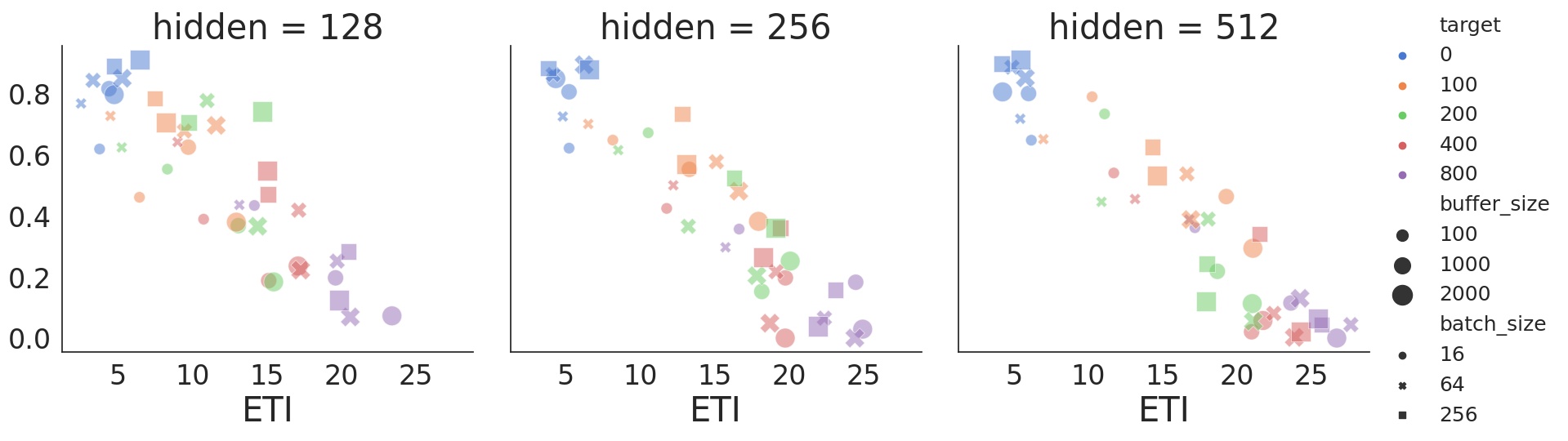}
        \caption{Offline sample efficiency.}
    \end{subfigure}
    \begin{subfigure}[t]{0.48\textwidth}
        \centering
        \includegraphics[width=\textwidth]{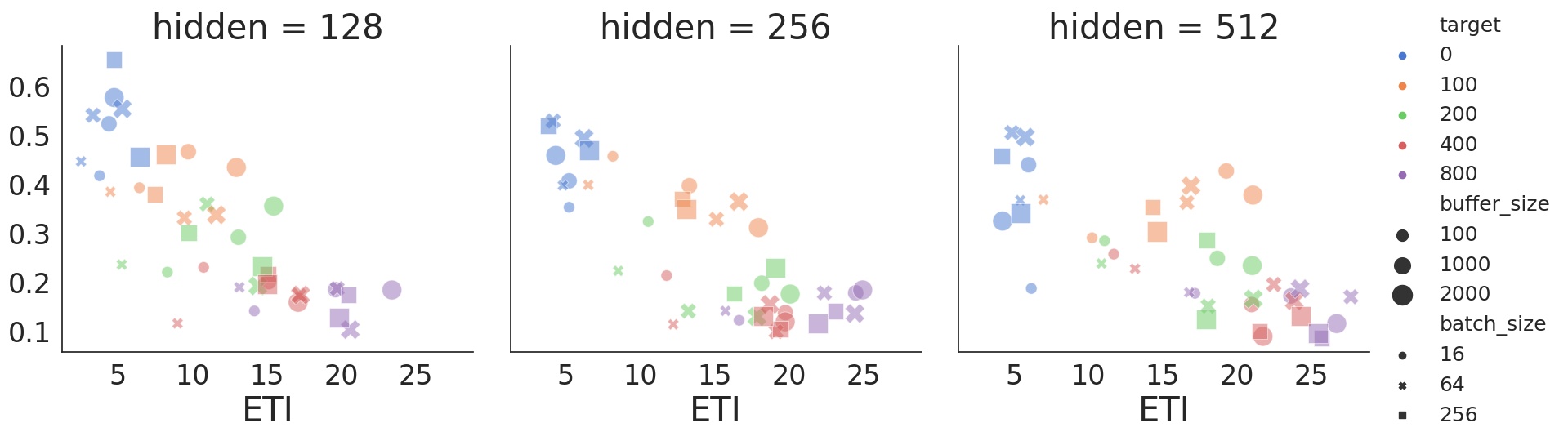}
        \caption{Online consecutive stable performance.}
    \end{subfigure}
    \begin{subfigure}[t]{0.48\textwidth}
        \centering
        \includegraphics[width=\textwidth]{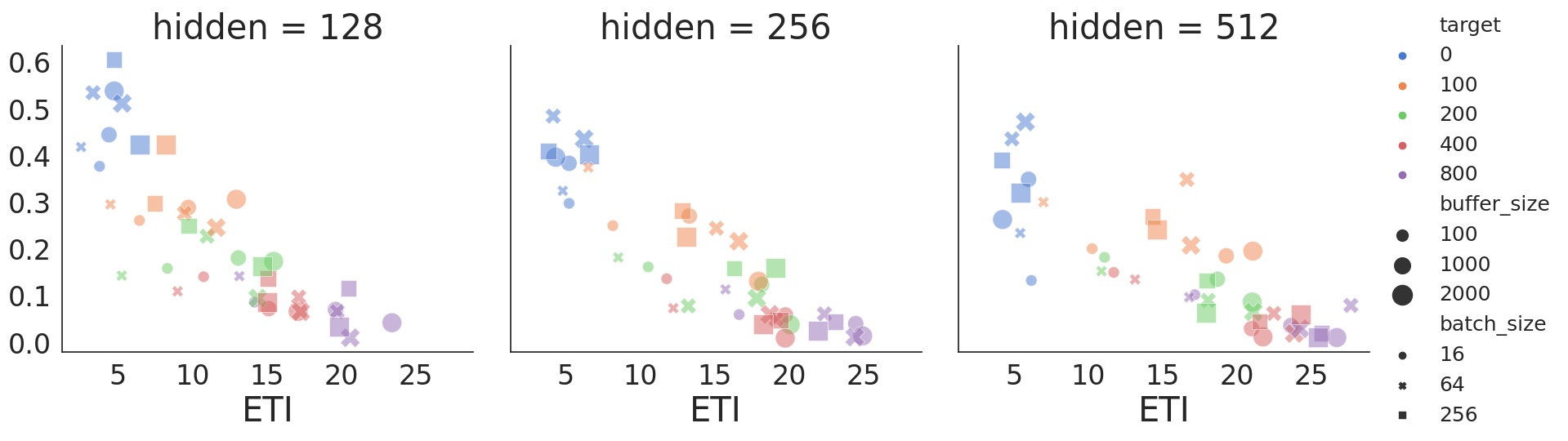}
        \caption{Offline consecutive stable performance.}
    \end{subfigure}
    \begin{subfigure}[t]{0.48\textwidth}
        \centering
        \includegraphics[width=\textwidth]{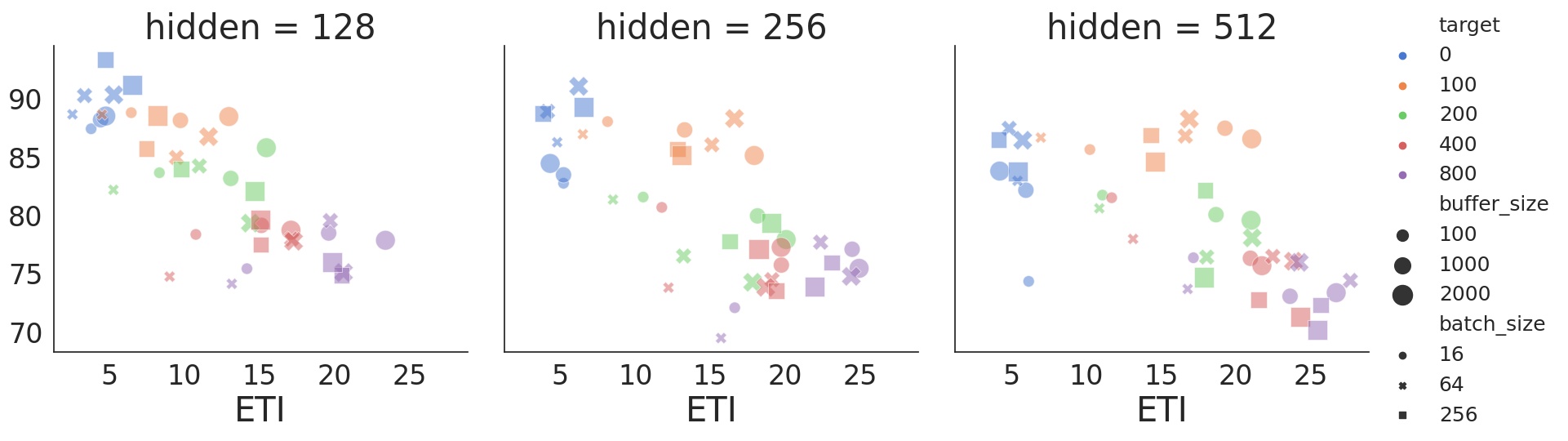}
        \caption{Online AER.}
    \end{subfigure}
    \begin{subfigure}[t]{0.48\textwidth}
        \centering
        \includegraphics[width=\textwidth]{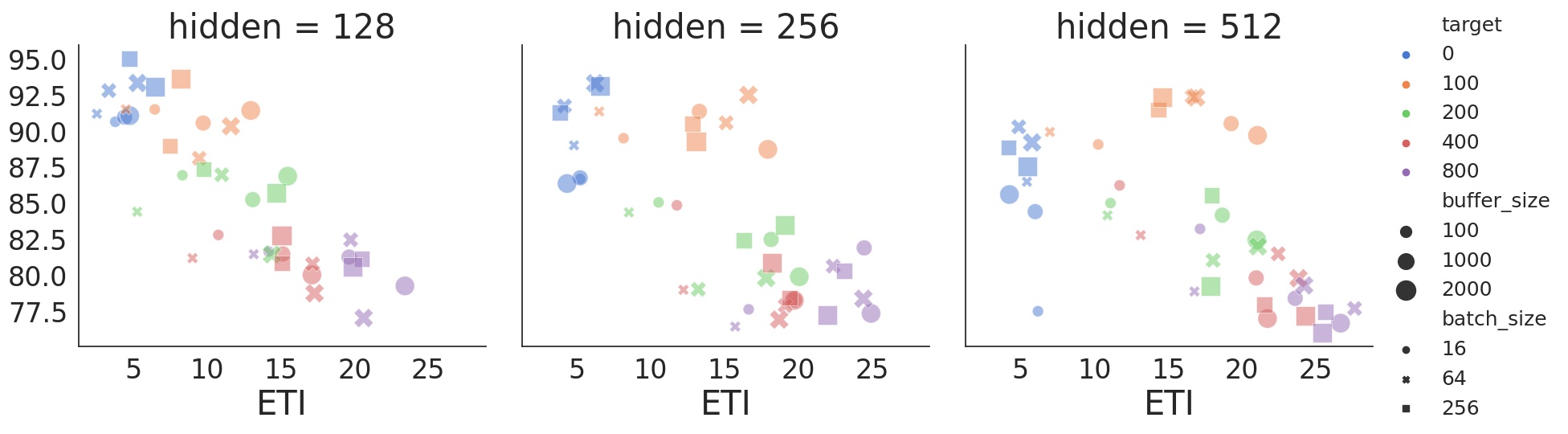}
        \caption{Offline AER.}
    \end{subfigure}
    \begin{subfigure}[t]{0.48\textwidth}
        \centering
        \includegraphics[width=\textwidth]{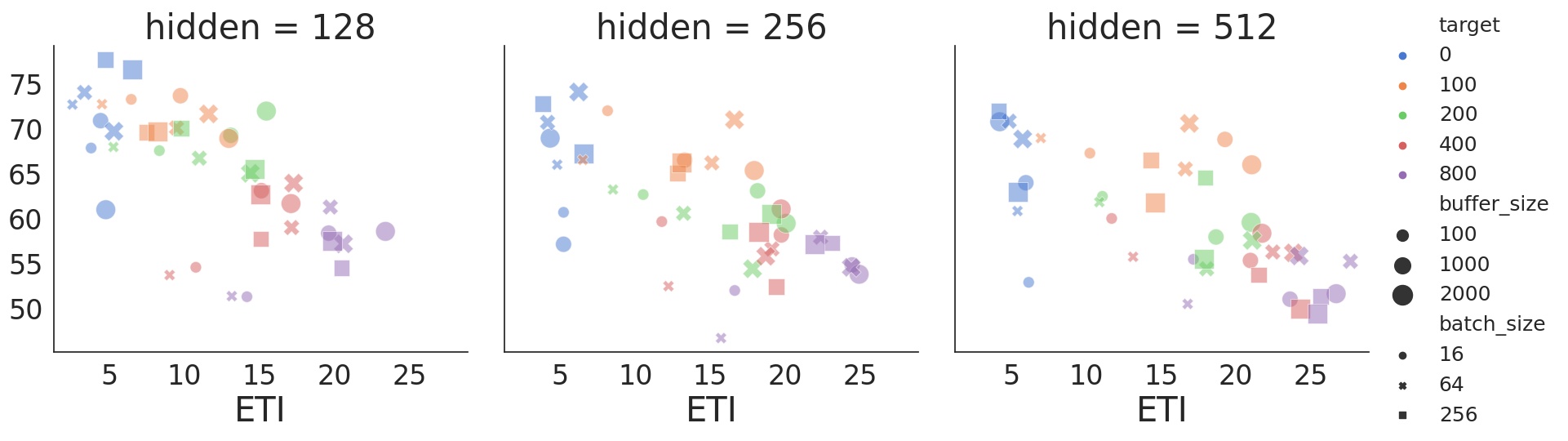}
        \caption{Online stable AER ($\beta=0.5$).}
    \end{subfigure}
    \begin{subfigure}[t]{0.48\textwidth}
        \centering
        \includegraphics[width=\textwidth]{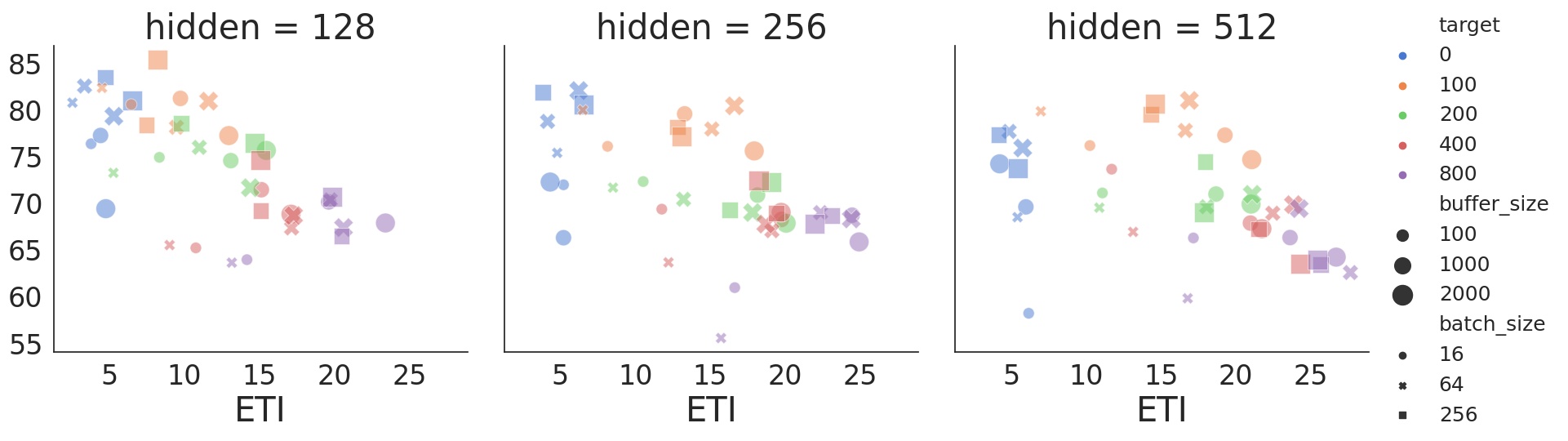}
        \caption{Offline stable AER ($\beta=0.5$).}
    \end{subfigure}
    \begin{subfigure}[t]{0.48\textwidth}
        \centering
        \includegraphics[width=\textwidth]{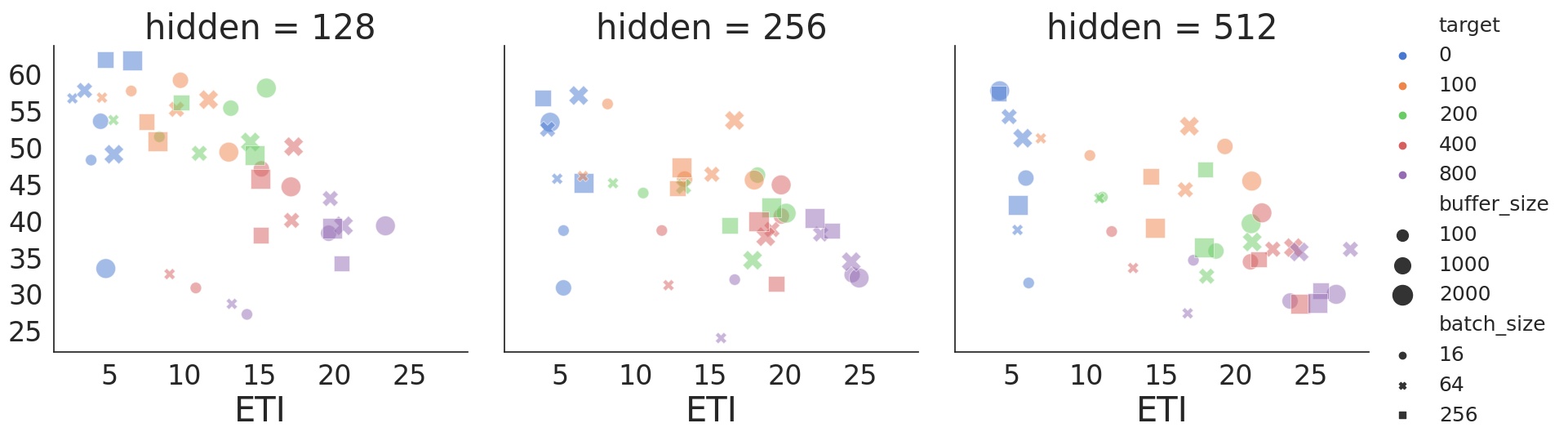}
        \caption{Online stable AER ($\beta=0$).}
    \end{subfigure}
    \begin{subfigure}[t]{0.48\textwidth}
        \centering
        \includegraphics[width=\textwidth]{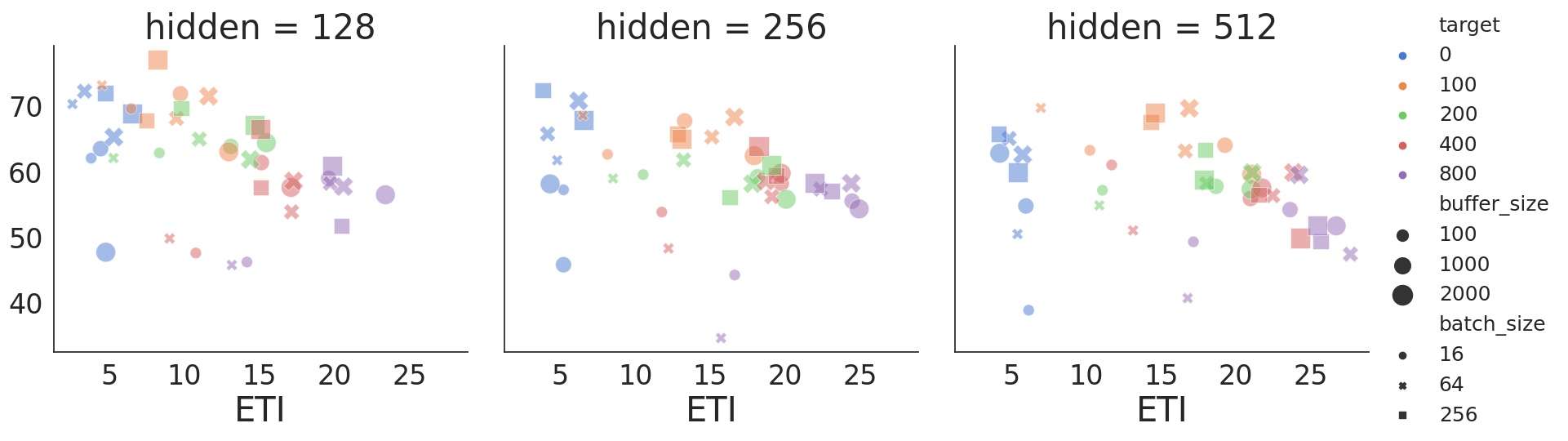}
        \caption{Offline stable AER ($\beta=0$).}
    \end{subfigure}
    \caption
    {
        ETI vs performance measures in Cart-pole, for a variety of Deep RL agents.
    }
    \label{eti_cartpole}
\end{figure*}

\begin{figure*}[ht]
    \centering
    \begin{subfigure}[t]{0.48\textwidth}
        \centering
        \includegraphics[width=\textwidth]{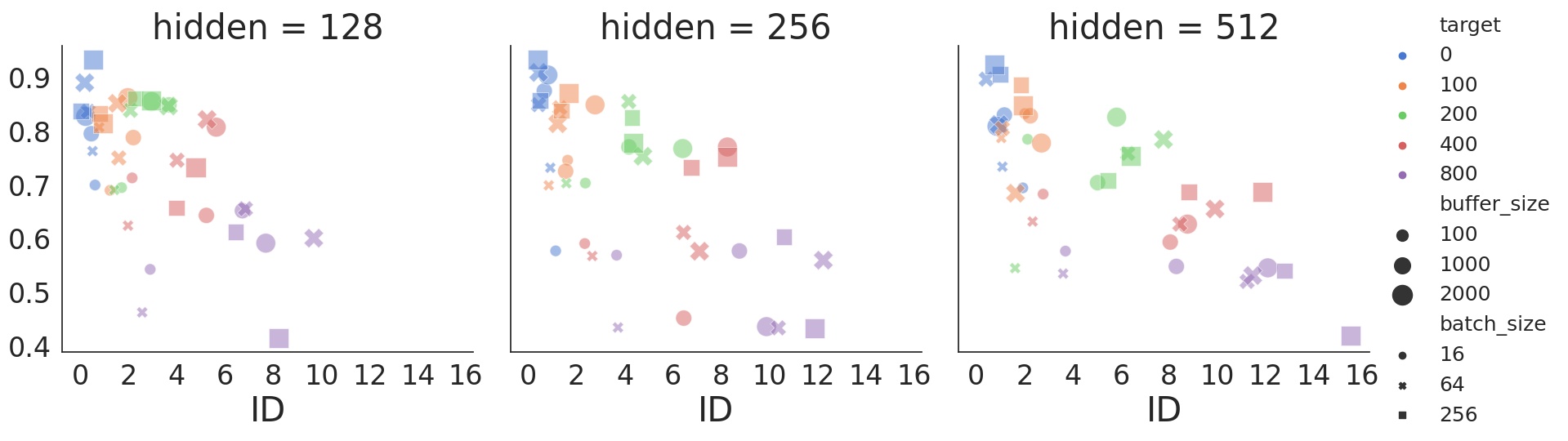}
        \caption{Online sample efficiency.}
    \end{subfigure}
    \begin{subfigure}[t]{0.48\textwidth}
        \centering
        \includegraphics[width=\textwidth]{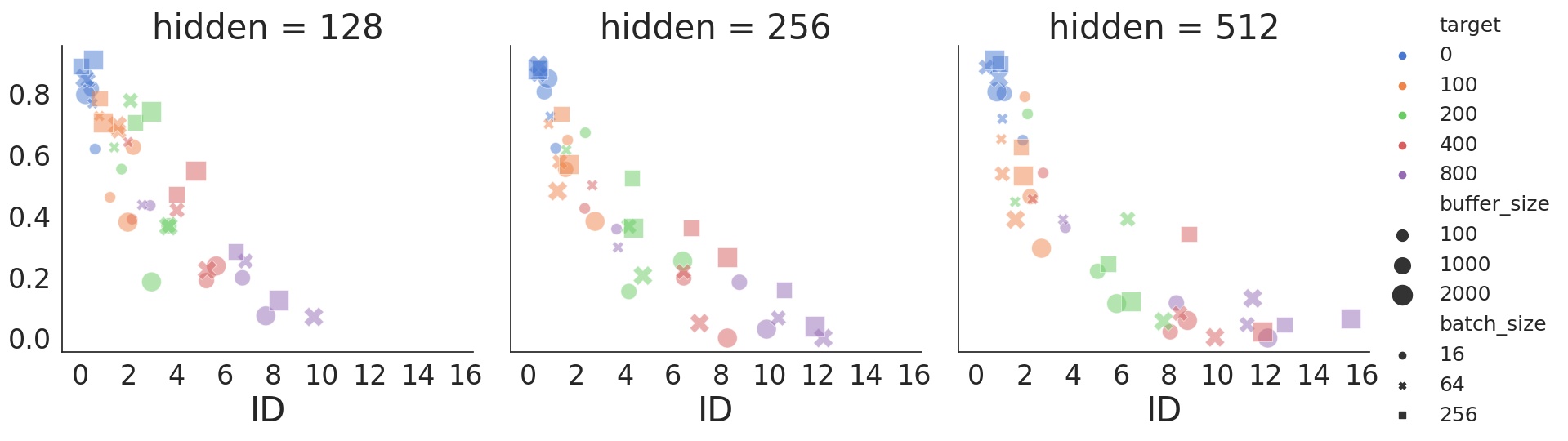}
        \caption{Offline sample efficiency.}
    \end{subfigure}
    \begin{subfigure}[t]{0.48\textwidth}
        \centering
        \includegraphics[width=\textwidth]{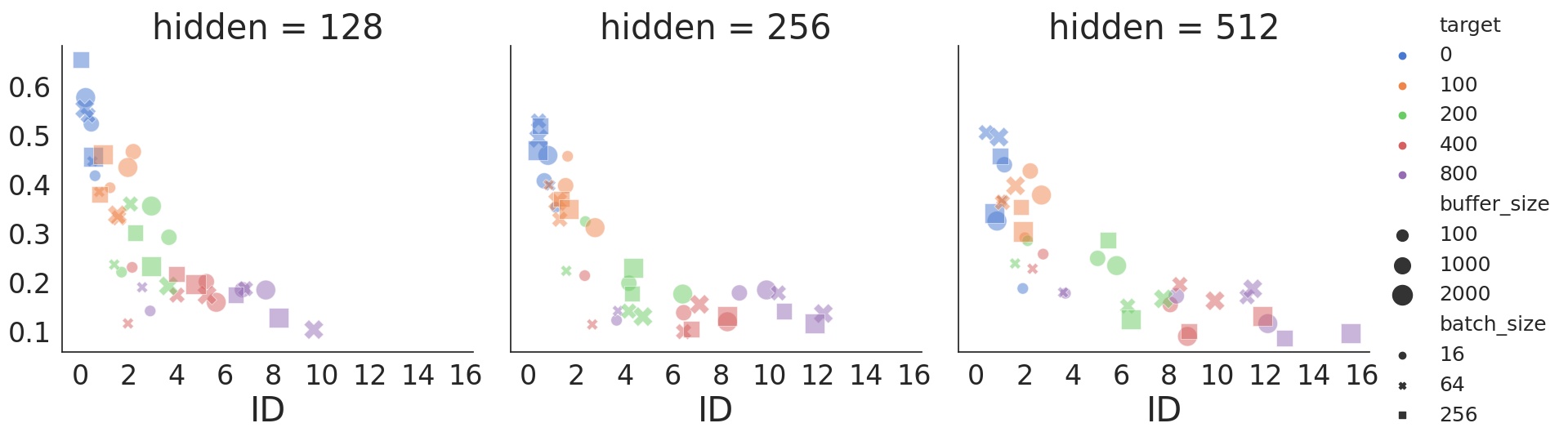}
        \caption{Online consecutive stable performance.}
    \end{subfigure}
    \begin{subfigure}[t]{0.48\textwidth}
        \centering
        \includegraphics[width=\textwidth]{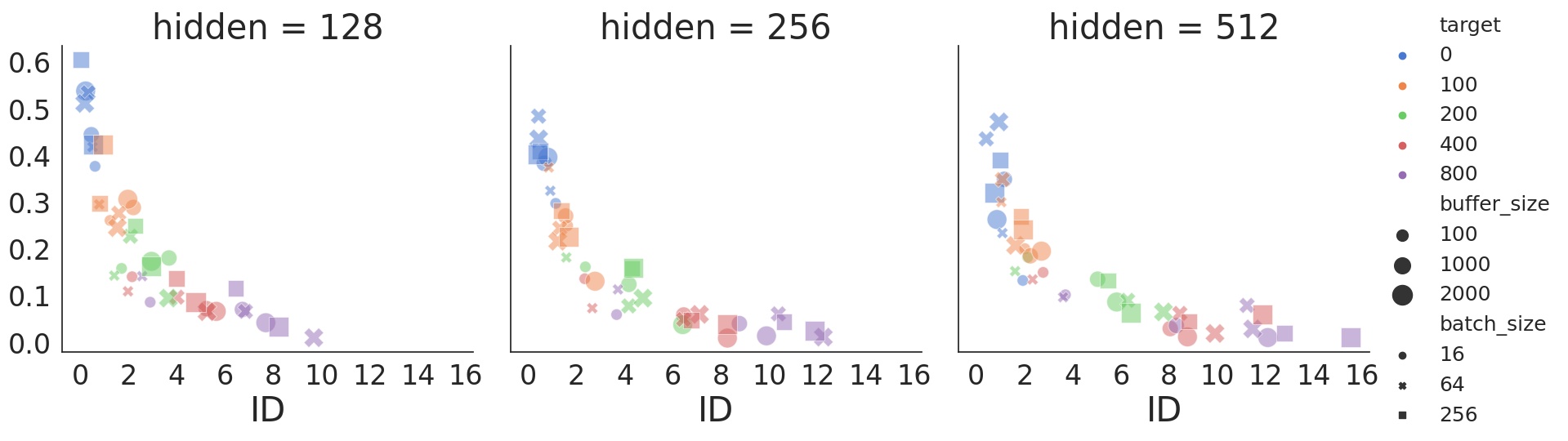}
        \caption{Offline consecutive stable performance.}
    \end{subfigure}
    \begin{subfigure}[t]{0.48\textwidth}
        \centering
        \includegraphics[width=\textwidth]{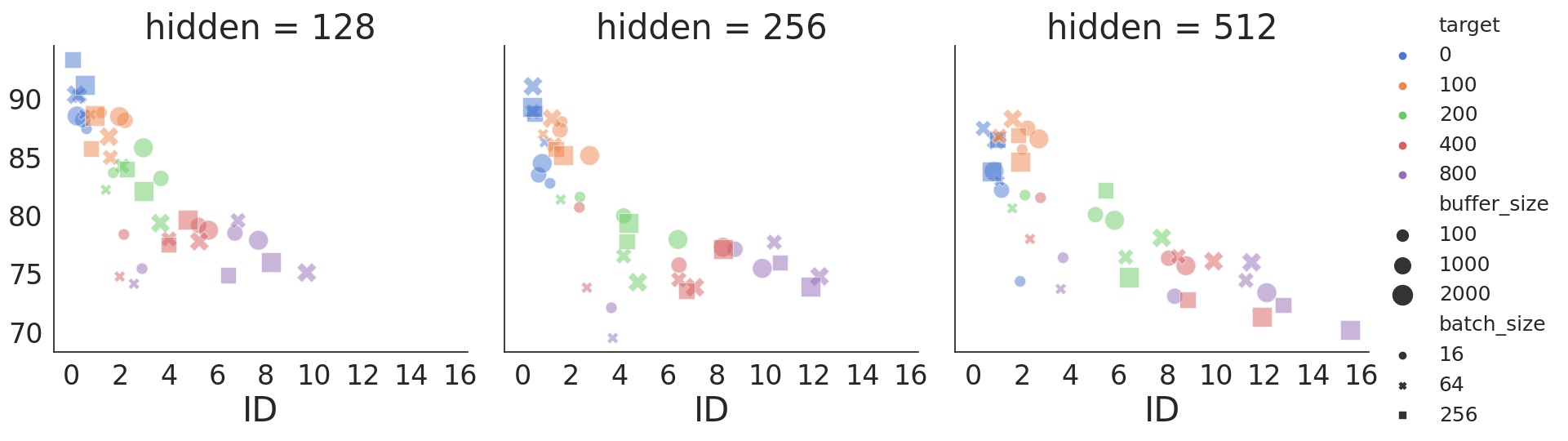}
        \caption{Online AER.}
    \end{subfigure}
    \begin{subfigure}[t]{0.48\textwidth}
        \centering
        \includegraphics[width=\textwidth]{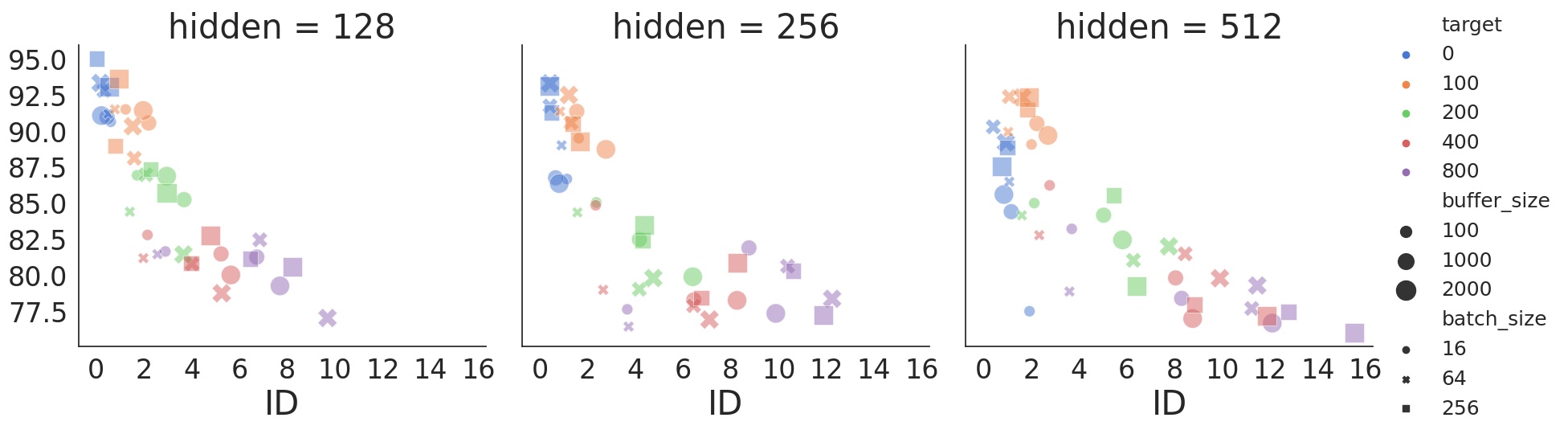}
        \caption{Offline AER.}
    \end{subfigure}
    \begin{subfigure}[t]{0.48\textwidth}
        \centering
        \includegraphics[width=\textwidth]{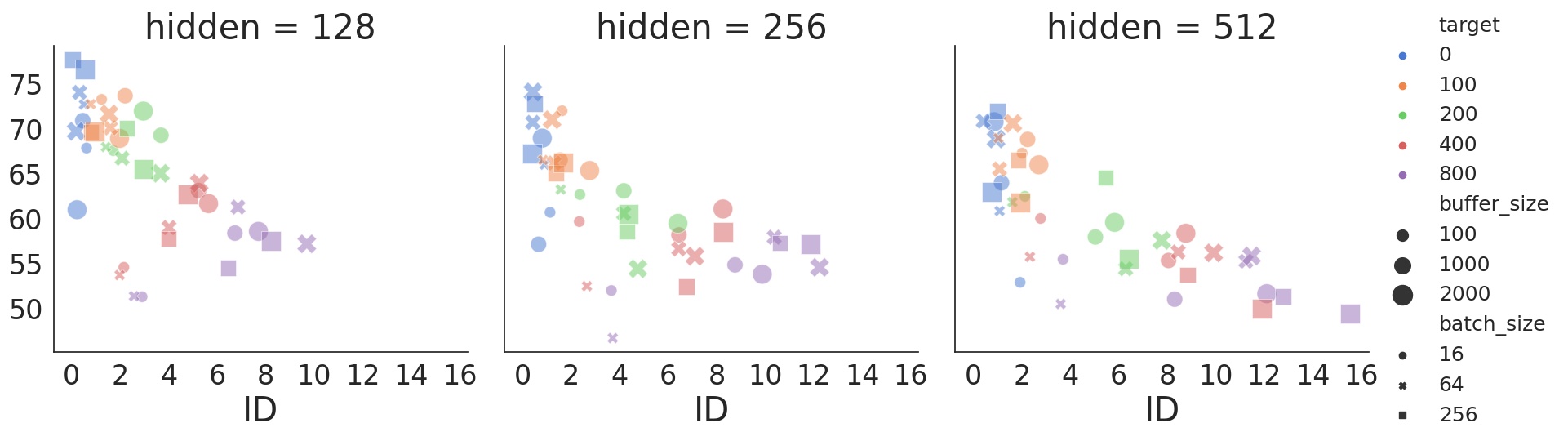}
        \caption{Online stable AER ($\beta=0.5$).}
    \end{subfigure}
    \begin{subfigure}[t]{0.48\textwidth}
        \centering
        \includegraphics[width=\textwidth]{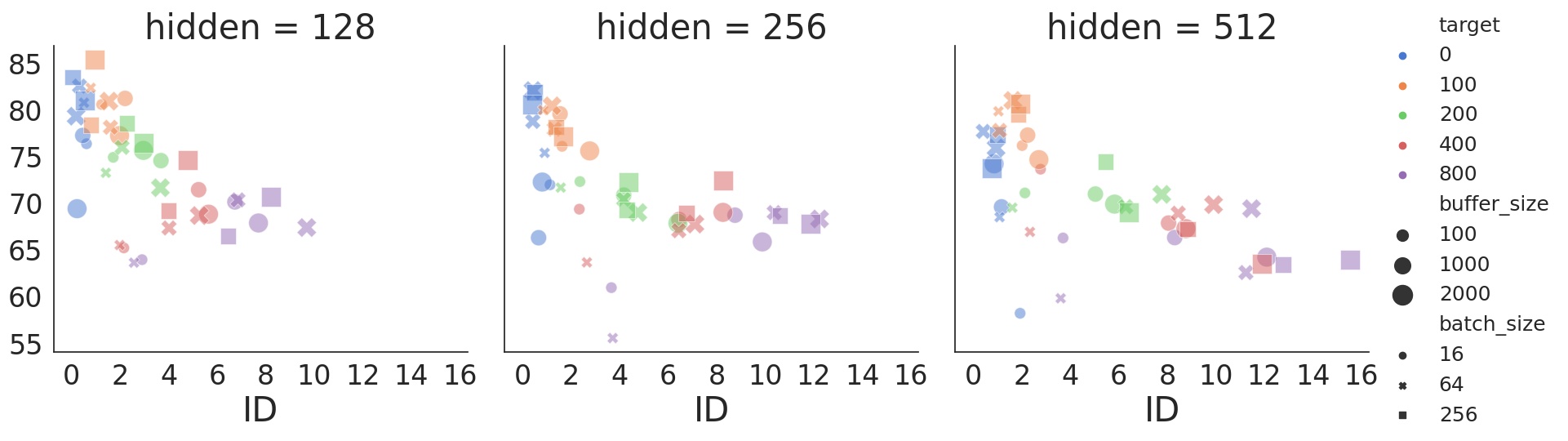}
        \caption{Offline stable AER ($\beta=0.5$).}
    \end{subfigure}
    \begin{subfigure}[t]{0.48\textwidth}
        \centering
        \includegraphics[width=\textwidth]{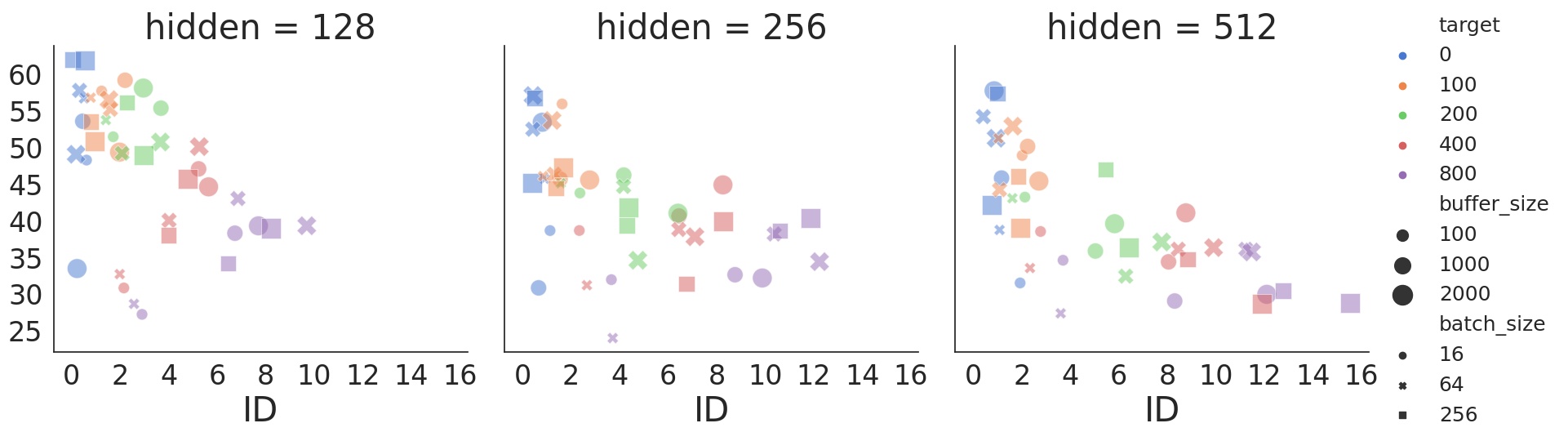}
        \caption{Online stable AER ($\beta=0$).}
    \end{subfigure}
    \begin{subfigure}[t]{0.48\textwidth}
        \centering
        \includegraphics[width=\textwidth]{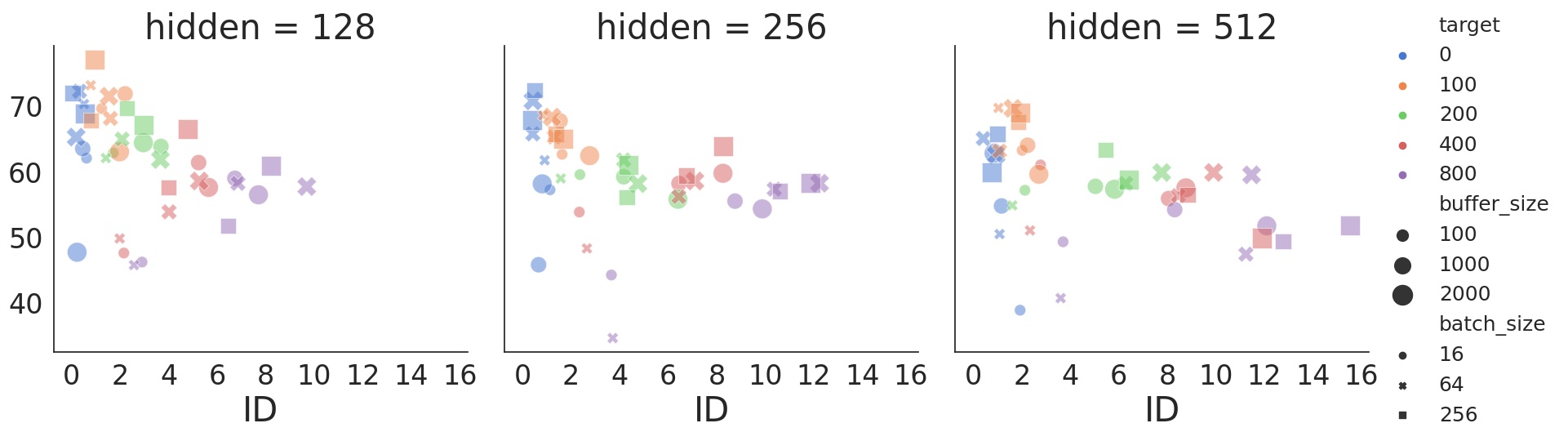}
        \caption{Offline stable AER ($\beta=0$).}
    \end{subfigure}
    \caption
    {
        Interference Dispersion vs performance measures in Cart-pole, for a variety of Deep RL agents.
    }
    \label{id_cartpole}
\end{figure*}

\begin{figure*}[ht]
    \centering
    \begin{subfigure}[t]{0.48\textwidth}
        \centering
        \includegraphics[width=\textwidth]{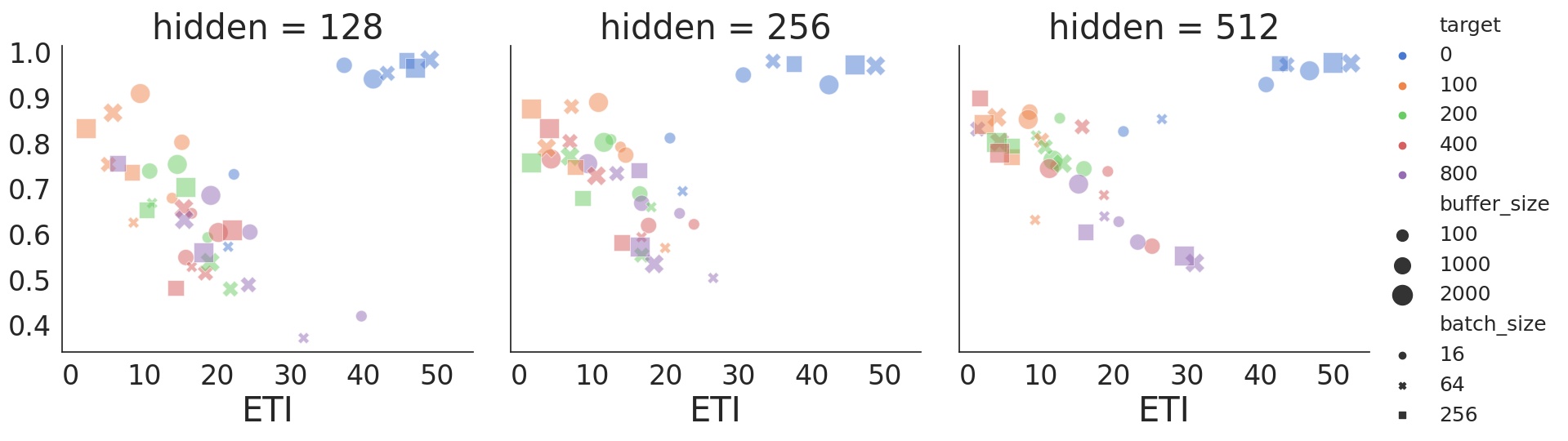}
        \caption{Online sample efficiency.}
    \end{subfigure}
    \begin{subfigure}[t]{0.48\textwidth}
        \centering
        \includegraphics[width=\textwidth]{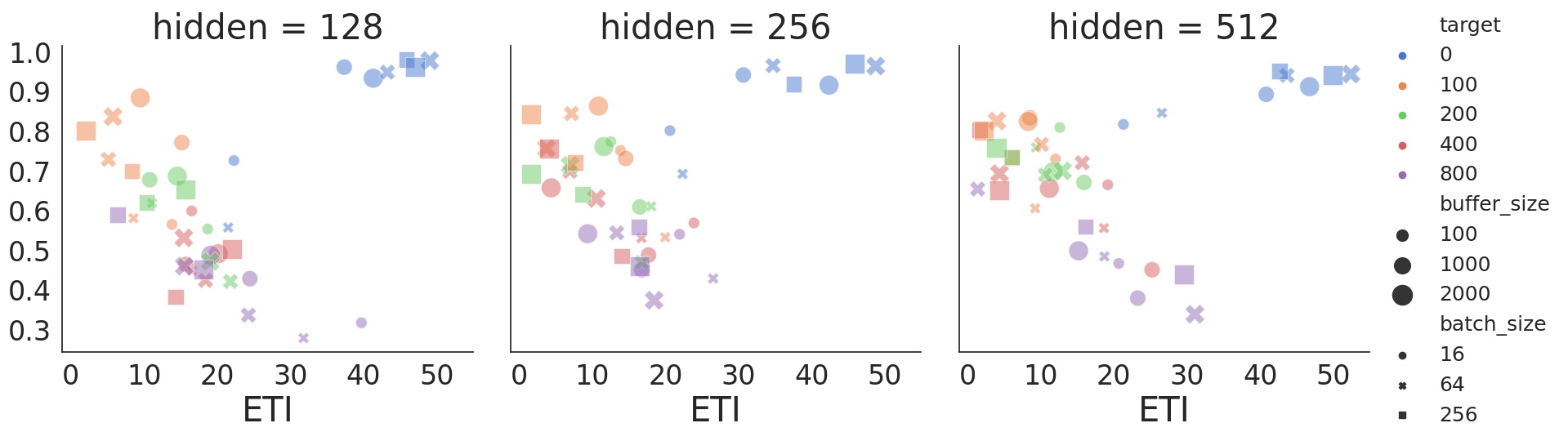}
        \caption{Offline sample efficiency.}
    \end{subfigure}
    \begin{subfigure}[t]{0.48\textwidth}
        \centering
        \includegraphics[width=\textwidth]{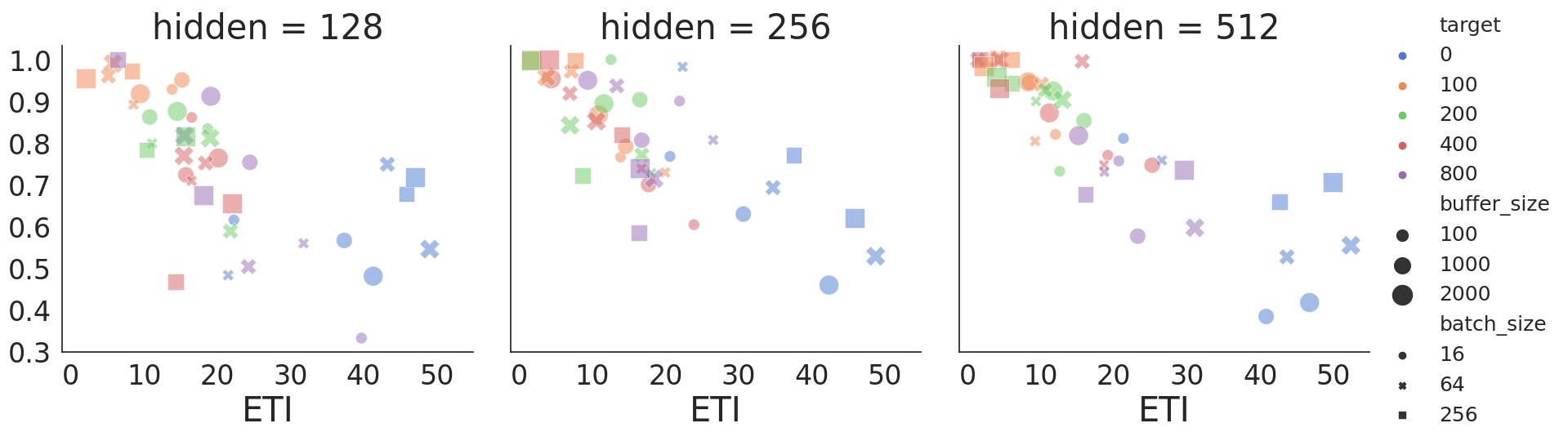}
        \caption{Online consecutive stable performance.}
    \end{subfigure}
    \begin{subfigure}[t]{0.48\textwidth}
        \centering
        \includegraphics[width=\textwidth]{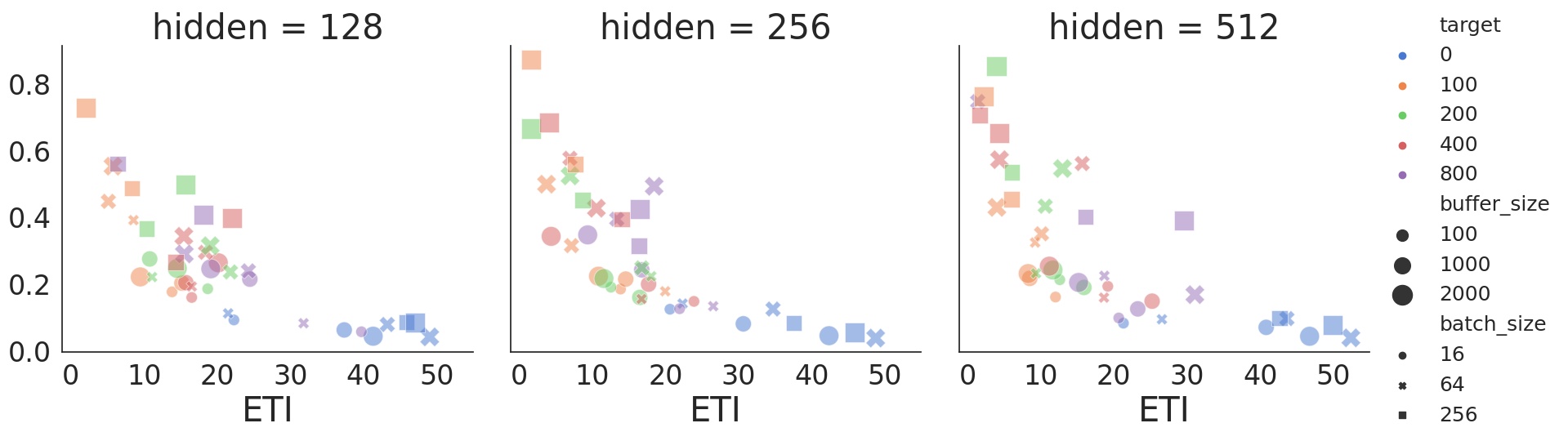}
        \caption{Offline consecutive stable performance.}
    \end{subfigure}
    \begin{subfigure}[t]{0.48\textwidth}
        \centering
        \includegraphics[width=\textwidth]{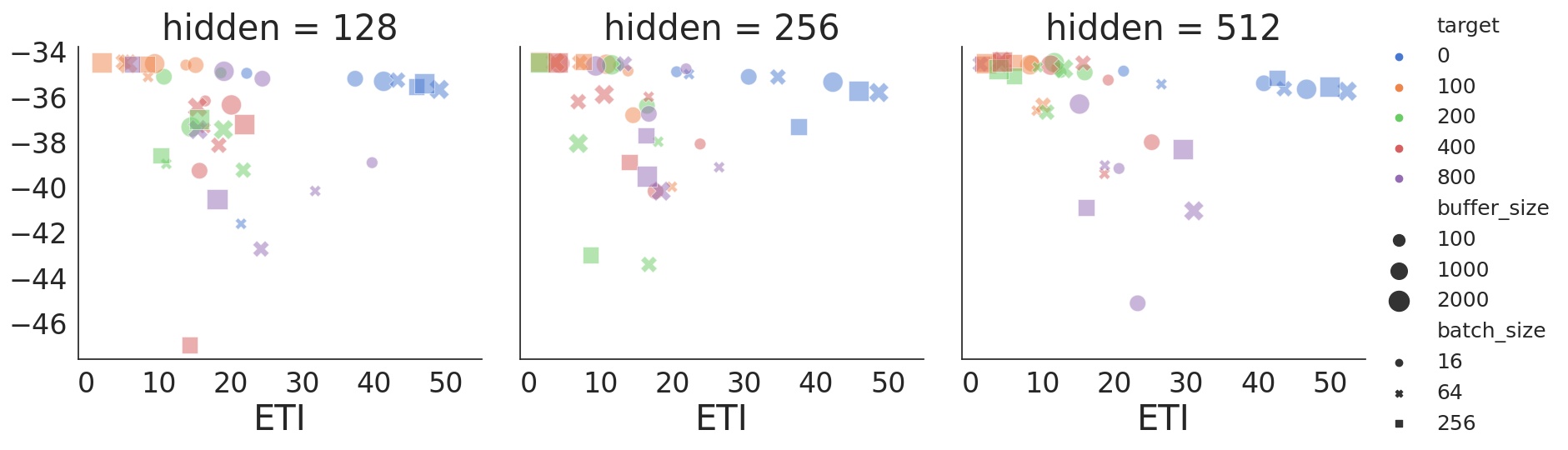}
        \caption{Online AER.}
    \end{subfigure}
    \begin{subfigure}[t]{0.48\textwidth}
        \centering
        \includegraphics[width=\textwidth]{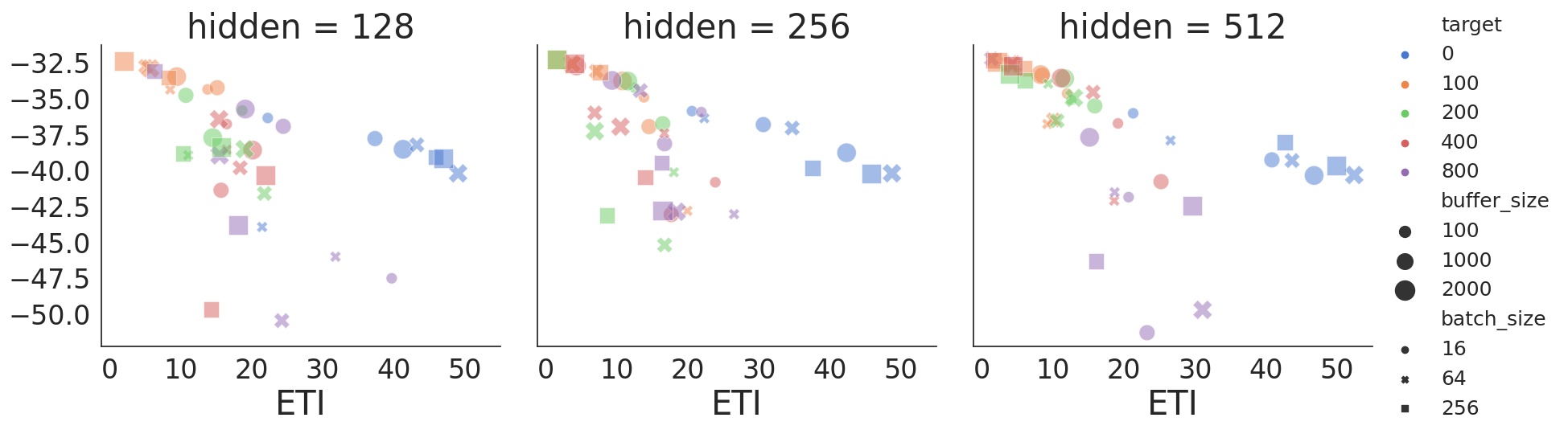}
        \caption{Offline AER.}
    \end{subfigure}
    \begin{subfigure}[t]{0.48\textwidth}
        \centering
        \includegraphics[width=\textwidth]{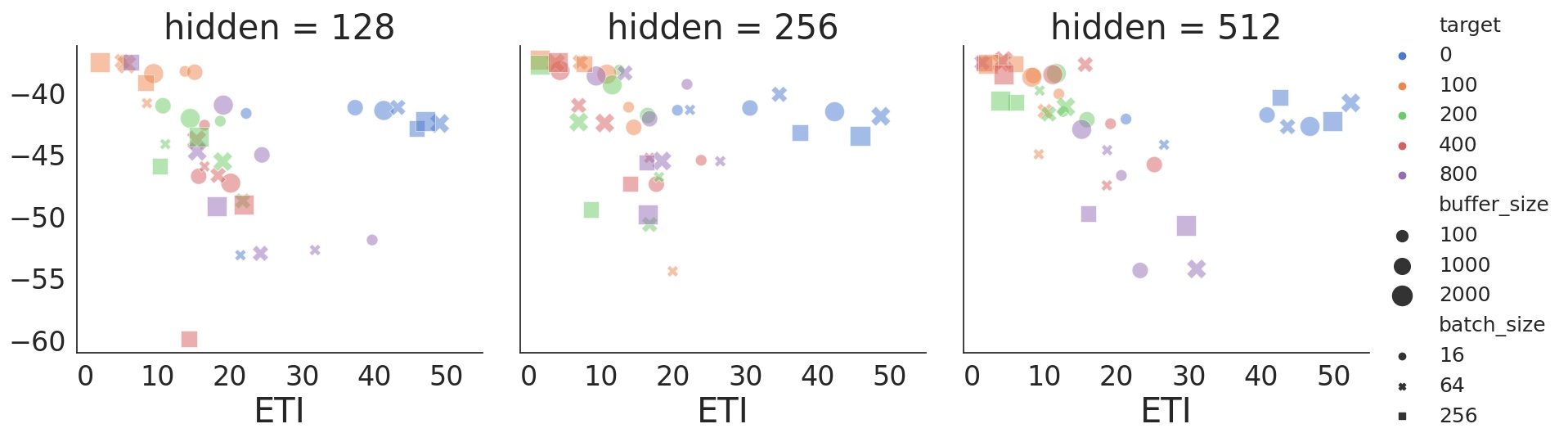}
        \caption{Online stable AER ($\beta=0.5$).}
    \end{subfigure}
    \begin{subfigure}[t]{0.48\textwidth}
        \centering
        \includegraphics[width=\textwidth]{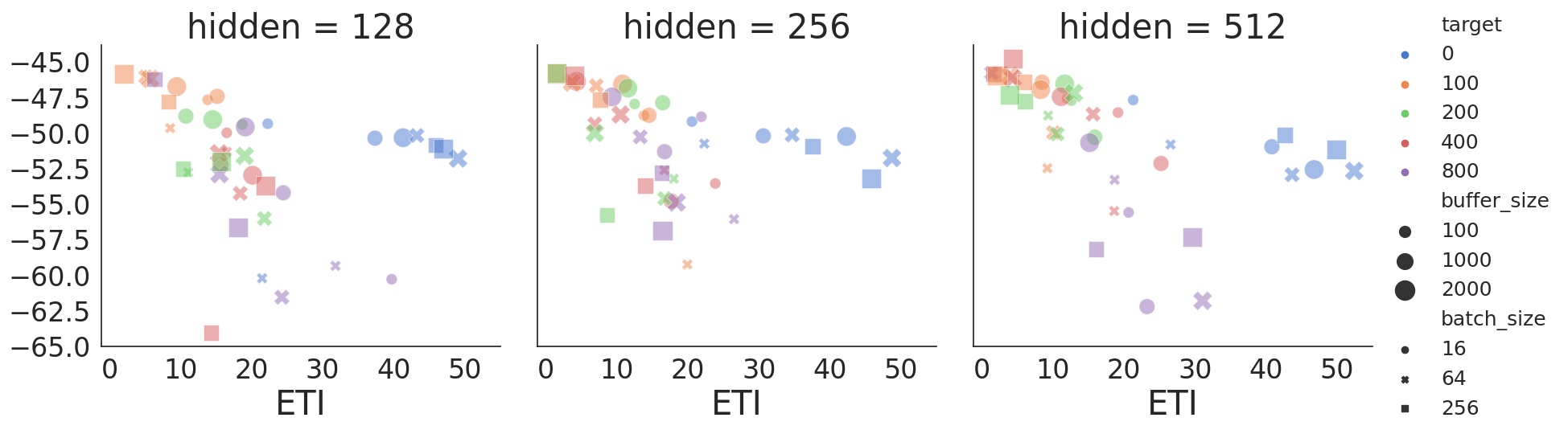}
        \caption{Offline stable AER ($\beta=0.5$).}
    \end{subfigure}
    \begin{subfigure}[t]{0.48\textwidth}
        \centering
        \includegraphics[width=\textwidth]{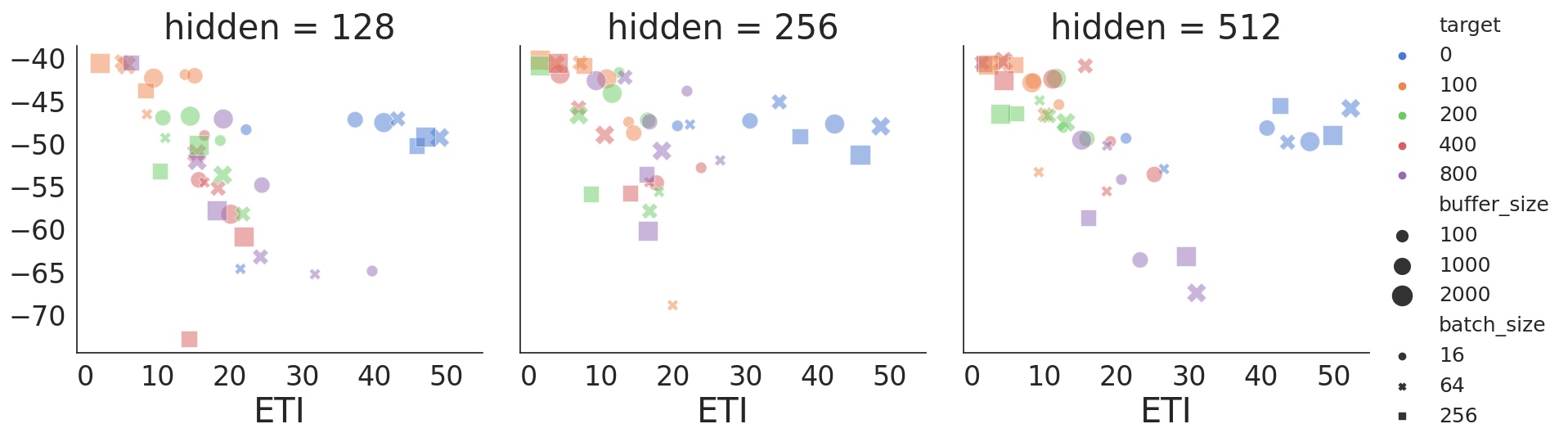}
        \caption{Online AER ($\beta=0$).}
    \end{subfigure}
    \begin{subfigure}[t]{0.48\textwidth}
        \centering
        \includegraphics[width=\textwidth]{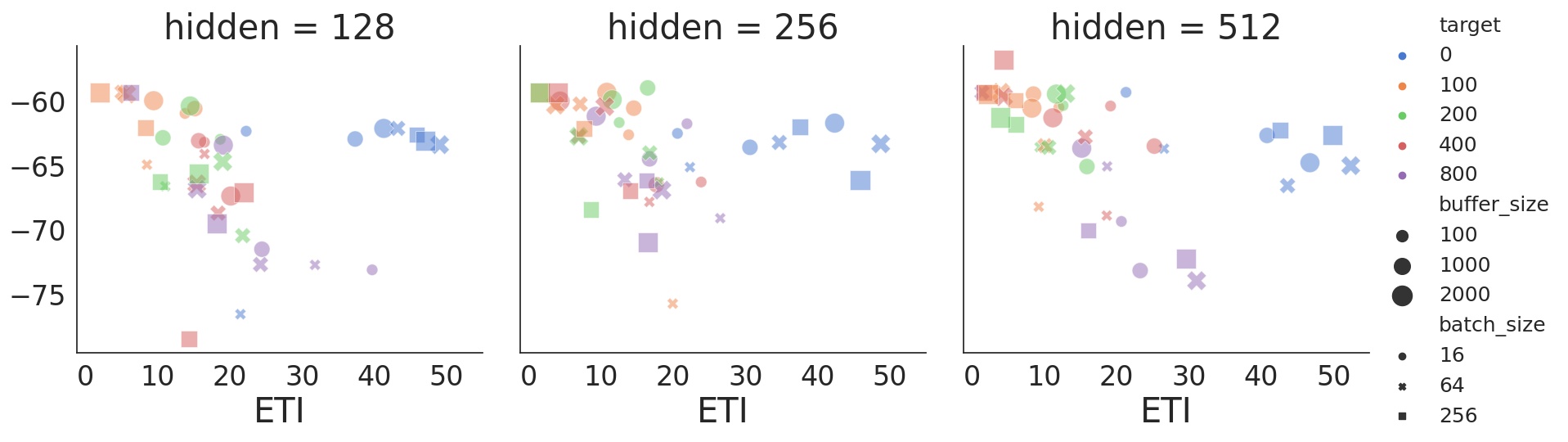}
        \caption{Offline AER ($\beta=0$).}
    \end{subfigure}
    \caption
    {
        ETI vs performance measures in Two-Rooms, for a variety of Deep RL agents.
    }
    \label{eti_tworooms}
\end{figure*}